\documentclass[10pt,conference]{IEEEtran}
\usepackage{blindtext, graphicx}
\ifCLASSINFOpdf
\else
\fi

\usepackage{graphicx}
\usepackage{float}
\usepackage{tikz}
\usepackage{amsmath,amsthm}
\usepackage{amssymb}
\usepackage{algpseudocode}
\usepackage{algorithm}
\usepackage{epstopdf}
\newcommand*{\permcomb}[4][0mu]{{{}^{#3}\mkern#1#2_{#4}}}

\newcommand*{\comb}[1][-1mu]{\permcomb[#1]{C}}
\usepackage{multirow}

\usepackage{subfig}
\usepackage{balance}
\usepackage{color}

\newcommand{\todo}[1]{{\color{black} #1}}

\newtheorem{lemma}{Lemma}

\newtheorem{clm}{Claim}
\newtheorem{defn}{Definition}

\newtheorem{assm}{Assumption}
\newtheorem{obj}{Objective}

\newtheorem{stp}{Step}

\usepackage{bbm}

\begin{document}
\title{Interference Power Bound Analysis of a Network of Wireless Robots}% to Provide Temporary Communication Backbone}
%Smart Wireless Adaptable Network of roBOTs
\author{\IEEEauthorblockN{Pradipta Ghosh}
\IEEEauthorblockA{Ming Hsieh Department of Electrical Engineering\\
University of Southern California\\
Los Angeles, California 90089\\
Email: pradiptg@usc.edu}
\and
\IEEEauthorblockN{Bhaskar Krishnamachari}
\IEEEauthorblockA{Ming Hsieh Department of Electrical Engineering\\
University of Southern California\\
Los Angeles, California 90089\\
Email: bkrishna@usc.edu}}


%Telephone: (800) 555--1212, Fax: (888) 555--1212}


\maketitle

\begin{abstract}
% In this paper, we focus on a relatively unexplored problem of determining the number of nodes to deploy in a 
We consider a fundamental problem concerning the deployment of a wireless robotic network: to fulfill various end-to-end performance requirements, a ``sufficient'' number of robotic relays must be deployed to ensure that links are of acceptable quality. Prior work has not addressed how to find this number. We use the properties of Carrier Sense Multiple Access (CSMA) based wireless communication to derive an upper bound on the spacing between any transmitter-receiver pair, which directly translates to a lower bound on the number of robots to deploy. We focus on SINR-based performance requirements due to their wide applicability. Next, we show that the bound can be improved by exploiting the geometrical structure of a network, such as linearity in the case of flow-based robotic router networks. Furthermore, we also use the bound on robot count to formulate a lower bound on the number of orthogonal codes required for a high probability of interference free communication. We demonstrate and validate our proposed bounds through simulations.
% Furthermore, we show that, while the SINR bounds of generic dense network are always valid, in our target robotic router network context, we can formulate tighter bounds on the SINR by exploiting the network structure, thereby leading to a significant $10-15\%$ improvement in the estimated number of robots to be deployed.
% a arbitrary fo, in order to properly choose the design parameters such as the transmitter power and spacing between the nodes (length of the links).  
% Furthermore, the SIR sufficient number of robots to be deployed in the field.
% The number of robots that need to be deployed depends on the performance requirement which thereafter relies on certain threshold of Signal to Interference plus Noise Ratio (SIR). So in order to estimate the number, we need an estimate of the worst case SIR at any node such that the requirements are always fulfilled. 
% While our main focus is on the interference estimation for a robotic network of router nodes, the SINR estimation methods can be easily morphed to help in deployment of a wireless network with certain performance requirements. 
\end{abstract}

% % Note that keywords are not normally used for peerreview papers.
% \begin{IEEEkeywords}
% Interference, Signal to Interference plus Noise Ratio (SINR), Carrier Sense Multiple Access (CSMA), Collision Avoidance (CA), Robotic Router
% \end{IEEEkeywords}

% \fi
\IEEEpeerreviewmaketitle

\section{Introduction}
In the field of Robotics and Automation, one of the emerging area of research is focused on the applicability of a wireless network of robots to create a temporary communication backbone between a set of communication endpoints with no or limited connectivity~\cite{williams2014route}. In these contexts, the robots act as relay nodes to form wireless communication paths between the communication endpoints. The application of this field of research ranges from fire fighting~\cite{penders2011robot} and underground mining \cite{thrun2004autonomous} to supporting temporary increase in the communication demands or creating a secure mesh network for clandestine operations~\cite{nguyen2003autonomous}. 
% One of the well known problem in such application contexts is to properly place the robots such that all the links fulfil certain performance criteria such as maximum allowed bit error rate (BER) or minimum supported data rate. 
To the best of our knowledge, one of the \todo{unexplored} problem in this context is to determine the number of robots to deploy such that all the links can maintain certain acceptable link qualities, such as maximum allowed bit error rate (BER) or minimum supported data rate, in presence of fading and shadowing.
Interestingly, most of these link quality metrics are known to be directly related to the Signal to Interference plus Noise ratio (SINR) of the links. Now, the SINR value of a link depends on the spacing between the transmitter and receiver of the link as well as the locations of the interfering nodes. Thus, an offline characterization of SINR values as a function of the maximum allowed inter-node distance is required to properly select the number of nodes to be deployed and to properly place the nodes across a deployment region. Moreover, the presence of CSMA/CA among the robots needs to be taken into account for more practical estimation.

% The number of robots to be deployed directly depends on the spacing between the transmitter and receiver of the link which, thereafter, depends on the communication performance requirements of the links.
% that can be directly mapped into Signal to Interference plus Noise ratio (SINR) requirements of the links. 
% Therefore, an offline characterization of SINR values as functions of the maximum allowed inter-node distances is required to properly select the number of nodes to be deployed and to properly place the nodes across a deployment region. 

% before actual deployment, with limited knowledge of the wireless environment as well as the final network topology in a real deployment scenario. This information will guide the system designers to choose the parameters such as transmitter power ($P_t$) and Radio Signal Strength thresholds for differentiating noise, interference and actual signal transmission, to guarantee the target performance requirements. 
% before the deployment of the network, with limited knowledge of the wireless environment in the deployment scenario, we need to estimate the worst possible SIR at each node so that the network is prepared to support it. 

In our venture for a generic model to estimate the number of robots to deploy (by estimating the maximum allowed inter-node distance to maintain the target SINR), we explored the existing literature in search for a proper model of interference and Signal to Interference plus Noise Ratio (SINR) range analysis in a CSMA/CA based wireless network. There exist a large body of works that characterize the mean interference power distribution in CSMA networks  (\cite{haenggi2011mean,ganti2009interference}) by employing the concepts of point process such as Poisson Point process, Ma{t'}ern hard core process and Simple Sequential Inhibition\cite{busson2014capacity}. \emph{The basic idea of this class of work is to represent the locations of the interferers as spatial point processes, more specifically, hard core point processes where the nodes fulfil a criterion of being certain distance apart to take into account CSMA among themselves.} Through application of different point process properties such as thinning and superpositions, researchers (\cite{haenggi2011mean,elsawy2012modeling,ganti2009interference,busson2009point}) estimated the probability distributions of the mean interference powers in the presence of CSMA/CA. Interested readers are referred to \cite{cardieri2010modeling} for a detailed survey on this class of works.
% \textbf{While all these works are very thorough and important, they can not be easily converted to our advantage for designing an interference aware wireless robotic router system as we require a good estimate on the number of robots to be deployed, by exploiting the worst case Interference or SINR to be expected by a robot, rather than the mean values.}
% % , as these distributions do not help us to guarantee certain performance bounds or estimate the number of devices needed to deploy.
% In fact, to our astonishment, there is a serious lack of Interference/SINR bound analysis that can be directly applied in our robotic router system design process.%  real system design tool 
Among the other class of works, the work of Hekmat and Van Mieghem~\cite{hekmat2004interference} is the most relevant to us. They demonstrated that the interference power in the presence of CSMA is actually upper bounded and can be best estimated by use of hexagonal lattice structure. 
% Since then, to the best of our knowledge, no other researcher have tried to extend this bound's application in the design and deployment of a real system. 
\emph{However, this work as well as most of the other works include some assumptions such as the receiver being located at the center of a contention region, which is only acceptable if the devices follow the 802.11 RTS/CTS standards~\cite{bianchi2000performance}. Interestingly, in practice, very few commercially available products actually employ the RTS/CTS mechanism. Furthermore, the Internet of Things (IoT) and Wireless Sensor Network (WSN) standard 802.15.4, which is also a standard choice for robotic network platforms, does not use RTS/CTS mechanism, in order to avoid inefficiencies.} Thus, it is actually the transmitter that employs the CSMA and should be located at the center of the contention region, whereas, the receiver is free to be anywhere inside the transmitter's communication range. In such cases, the SINR and the interference mean values as well as the bounds for a link are, in fact, functions of the separation distance ($d$) between the endpoints of the link. \textbf{However, none of the existing works try to characterize the SINR or the interference as a function of the separation distance ($d$), which is crucial for the number of robot estimations. }
In this paper, we modify the bounds proposed in~\cite{hekmat2004interference} and flesh out details of applying the modified bounds to estimate the number of robots to be deployed to satisfy the communication performance goals. Note that, in the rest of the paper, we focus on interference limited networks and, thereby, ignore the effect of noise and focus on Signal to Interference Ratio (SIR) instead of SINR.
% first revisit this work in order to extend it as a system design tool.

% On the other hand, in the context of a static wireless network deployment in scenarios with a large number of existing wireless interfering nodes, the system designers goals are typically to fulfil the target performance criterion as well as to minimize the cost of deployment by avoiding unnecessary node deployments. The goal is to optimize the performance of each wireless link which thereafter dictates the number of nodes to be deployed to cover the area of interest. One can think of a standalone transmitter and receiver pair placement as a sub-problem of the wireless network deployment problem. Thus, even in this context, the characterization of SIR in terms of node placements will act as a great tool/information to the designers. 
\begin {table*}[t]
\footnotesize

\parbox{.45\linewidth}{
\centering
\caption{General Parameters}
\label{table:symbol}
\begin{tabular}{|c|c|}
    \hline
    Symbol & Description \\
     \hline
    $T$ & Transmitter \\
    \hline $X$ & Receiver\\
    \hline $d_{ij}$ & Distance between node $i$ and $j$\\
     \hline $d$ & Distance between $T$ and $X$ i.e., $d_{TX}$\\
    \hline $\eta$ & Path Loss Exponent\\
    \hline $\psi \sim \mathcal{N}(0,\sigma^2)$ & Log normal Fading Noise with variance $\sigma^2$\\
    \hline $P_t$ & Transmitted Signal Power\\
    \hline $P_r$ & Received Signal Power\\
    \hline $P_\mathcal{I}$ & Received Interference Power \\
    \hline $\mathcal{I}^C$ & Interference Set Cover \\
    \hline $M$ & Number of Flows \\
     \hline
\end{tabular}

}
\hfill
\parbox{.45\linewidth}{
\caption{System Parameters}
\label{table:system}
\begin{tabular}{|c|c|}
    \hline
    Symbol & Description \\
     \hline
    $SIR_{th}$ & The Target Minimum SIR\\ \hline
    $SIR_X(d)$ & Minimum Achievable SIR at $X$ for $d$ separation \\ \hline
    $D_1$ & Contention Region Outer Radius\\      \hline 
    $D_2$ & Transition Region Outer Radius\\  \hline 
    $P_t$ & Transmitter Power\\  \hline 
    $\gamma$ & Required Probability of $SIR \geq SIR_{th}$ \\ \hline
    $\kappa$ & minimum probability of interference \\
    &free communication \\ \hline
    $d_{max}$ & maximum distance allowed between $T$ and $X$ \\ \hline
    $N_{\mathcal{O}}$ & Number of Orthogonal Codes\\ \hline
    $N_{\mathcal{I}}^{max}$ & Maximum Number of Interfering Nodes\\ 
     \hline
\end{tabular}

}
\end{table*}
In this paper, we \textbf{first} explain the concepts presented in \cite{hekmat2004interference} (for a general dense wireless network) as well as the impracticality of the bounds, followed by our proposed modified interference and SIR bounds as functions of the distance between a transmitter and a receiver, for any network that employs CSMA/CA. Through a set of simulation results we show that, with fading introduced in the model, we can form a stochastic bound as well, such that the probability of the real interference being higher than the bound is very low. This formulation helps any network designer to properly choose a maximum separation between the nodes and to properly place a set of nodes in any practical deployment. \textbf{Secondly,} we extend this bound one step further to determine a bound on the number of orthogonal codes to be used in order to guarantee a high probability of interference free communication. We also explore the bounds on interference power, if a fixed number of orthogonal codes are employed. \textbf{Thirdly,} we consider our application specific scenario of robotic router network to devise a better bound by applying the structure of the network. Through a set of simulation experiments we validate the bounds and show that the improved application specific bound significantly ($10\%-45\%$) decreases the required number of costly, resource constraint robots.

\section{Problem Description}
\label{sec:number}
% \subsection{Feasible SIR Bound on a Dense Wireless Network}
% \label{sec:general_form}
In this section, we detail our problem formulations. For compactness, we list the symbols used for base problem formulation in Table~\ref{table:symbol} and symbols related to our goals in Table~\ref{table:system}, respectively. Say, we have a transmitter node $T$ and a receiver node $X$ that are placed at $d$ distance apart, alongside with a larger number of interfering wireless nodes. Each node of this interference limited network (i.e., the interference dominates over noise) employs \emph{Channel Sense Multiple Access with Collision Avoidance} (CSMA/CA) ~\cite{rappaport1996wireless} for wireless media access and has a transmission power of $P_t$. The radio range of each node is subdivided into three regions, centered at the node's location: a circular \textbf{connected/contention region} of radius $D_1$, an annular \textbf{transition region} with inner radius $D_1$  and outer radius $D_2$ (including the boundaries), and a \textbf{disconnected region} which is the entire region outside the circle with radius $D_2 > D_1$; where the values of $D_1$ and $D_2$ depend on the actual RSSI thresholds of the devices used~\cite{zeng2014first}. Undoubtedly, in the presence of fading, the regions are not so nicely structured, nonetheless, can be approximated by proper choice of $D_1$ and $D_2$. 
Now, the CSMA restricts the transmissions from the nodes in the contention region of $T$, while the nodes in the transition region are aware of $T$'s transmission with very low probabilities and, therefore, are the potential interferers.
However, only a subset of the nodes in the transition region can be active simultaneously, due to CSMA among themselves, which requires any two simultaneous interferers to be at least $D_1$ distance apart. The interference power from the nodes in the disconnected region are considered insignificant. 

\begin{defn}
\label{def:isc}
A set of interfering nodes ($\mathcal{I}^C$) such that $D_2 \geq d_{ij} \geq D_1$ and $d_{iT} \geq{D_1} \ \forall \ i,j\in \mathcal{I}^C$, is referred to as an \textbf{Interference Set Cover}. 
\end{defn}

Now, there are four main objectives of this work as follows.
\begin{obj}
Find a mapping between $d$ and the minimum achievable SIR at $X$, $SIR_X(d)$.  \qed
\end{obj}
% the maximum number of interferers, $N_{\mathcal{I}}^{max}$ as well as 
\begin{obj}
Find the range, $0 < d \leq d_{max}$, such that the outage probability i.e., $\mathbb{P} (SIR_X(d)< SIR_{th}) < \gamma $ where $0 \leq \gamma \leq 0.5$ is the choice of the designer. \qed
\end{obj}

Now, one can employ a set of orthogonal codes to further restrict the interference in a CSMA network. In such cases, the maximum value of interference power decreases, based on the number of codes employed, possibly leading to near zero interference. In this context, our goal is as follows.
\begin{obj}
Characterize $SIR_X(d)$ as a function of the number of orthogonal codes ($N_{\mathcal{O}}$) employed for concurrent transmissions, and find a bound $N_{\mathcal{O}}'$ such that $\mathbb{P}(\mathbbm{1}_{\mathcal{I}0} = 1) \geq \kappa$ $\forall N_{\mathcal{O}} > N_{\mathcal{O}}'$, where the indicator function $\mathbbm{1}_{\mathcal{I}0}$ refers to interference free communication and $\kappa \geq 0.5$ is a designer choice.\qed
\end{obj}
% \begin{obj}
% Characterize the trade-off between the highest possible interference and attainable bandwidth (due to contention) in terms of the transmission power ($P_t$), for fixed value of $d$.\qed
% \end{obj}

For our SIR and Interference bound analysis, we consider two different scenarios in this paper. In the \textbf{first scenario}, the node pair in focus is placed in a \emph{``dense''} network, where a countably many \emph{uncontrollable wireless nodes} are co-located in the area of interest. 
% \subsection{Application Specific SIR Estimation for Robotic Wireless Network}
\textbf{Secondly}, we consider our target application of robotic router placement, where the goal is to place a set of robots such that they form multihop links between a set of maximum $M$ concurrent communication end-point pairs. This application context restricts the possible configuration of the interfering nodes within a class of network formations, such as straight line formation, that voids the earlier dense network assumption. 
At any time instance, we associate a set of routers with each flow $i \in \{1,2,\cdots M\}$ that form a chain between the communication endpoints. 
% The original performance requirement of the flows is to maintain a min data rate or bit rate, say $C_{min}$, which translates to the choice of $SIR_{th}$, \emph{thereby dictating the choice of $d_{max}$..
\emph{Thus, for a fixed set of communication endpoints of a flow $i$, the minimum number of nodes ($N_i^\mathcal{R}$) to be allocated to flow $i$ depends on $d_{max}$ which in turn controls the minimum number of nodes to be deployed, $N^\mathcal{R} \geq \sum_{i=1}^M  N_i^\mathcal{R}$.}
% Moreover, each link of these network should fulfil certain performance criteria such as maximum allowed BER, which we map into SIR requirements. \emph{In this context, the SIR to distance mapping reinforces the estimation of number of robots to be deployed such that regardless of final configuration of a deployment, the network can always meet the performance requirements.} 
\begin{obj}
Find a better and tighter bound on interference as well as SIR by exploiting the application specific restrictions on the network configurations. Next, analyze the improvement in the number of robots required, with this improved bound.\qed
\end{obj}

\section{Outline of the Proposed Solution}
\label{sec:proposed_sol}
In this section, we summarize our methodologies for achieving the target objectives while the details are discussed later on.
\subsection{Methodology for Mapping from $d$ to $SIR_X$}
\label{sec:summ_step}
For a fixed value of the separation distance $d$ between $T$ and $X$, we estimate the maximum feasible interference as well as minimum feasible SIR, by exploiting the geometry of the connectivity region and transition region. 
% First, we estimate the bound for a general dense network, followed by estimation of better bound for our target application of robotic router network. 
For received power modelling, we opt for the standard log normal fading model \cite{rappaport1996wireless}, where the received power is distributed log normally with mean power calculated using simple path loss model. Thus, the received power can be represented as:
\begin{equation}
\label{eqn:lognormal}
    P_r(d)=Q. P_t d^{-\eta}10^{\frac{\psi }{10}}
\end{equation}
where $Q$ is some constant.
Next, we introduce the following claim as our whole estimation process revolves around this claim.
\begin{clm}
\label{clm:mean}
In presence of Independent and Identically Distributed (I.I.D) fading noise, the Interference Set Cover (see Definition~\ref{def:isc}) with maximum mean power as well as maximum number of interferers will give us better stochastic bound than any other Interference Set Cover. 
\end{clm}
\begin{proof}[Justification]
This claim is justified by the fact that, if the fading noises are I.I.D, the Interference Set Cover with maximum number of nodes will give the highest variance. Thus, the Interference Set Cover with highest mean as well as highest number of nodes will be a better bound than any other Interference Set Cover.
\end{proof}

Now, the main steps for representing $SIR_X$ as a function of $d$ are as follows.
% \begin{itemize}
\begin{stp}
\label{stp:1}
    We first identify the \textbf{Interference Set Cover(s)} ($\mathcal{I}^C$) that will potentially give us the best estimate of the maximum feasible mean interference power, for a fixed $d$, using greedy algorithm.  
\end{stp}
\begin{stp}
    We estimate the maximum number of nodes in any \textbf{Interference Set Cover}, $N_\mathcal{I}^{max}$. 
\end{stp}
\begin{stp}
\label{stp:3}
    To get the maximum interference power, we add up the interference powers of the nodes of the Interference Set Covers selected in Step~\ref{stp:1}, according to Eqn~\eqref{eqn:lognormal}. Thus the total interference power at $X$ is a sum of log normal variables as follows.
    
    {\footnotesize
    \begin{equation}
    \label{eqn:pathloss1a}
        P_{\mathcal{I}^C}(d)= Q. \sum_{j \in \mathcal{I}^C} P_t d_{jX}^{-\eta} 10^{\frac{\psi }{10}}
    \end{equation}
    }
    % where $\zeta=\max \{1, \frac{N_\mathcal{I}^{max}}{|\mathcal{I}^C|} \}$ and $|.|$ denotes the cardinality of a set.
\end{stp}
\begin{stp}
    We multiply the interference power estimate in Step~\ref{stp:3} by a correction factor $\zeta=\max \{1, \frac{N_\mathcal{I}^{max}}{|\mathcal{I}^C|} \}$, where $|.|$ denotes the cardinality of a set, to account for the Interference Set Covers with less than  $N_\mathcal{I}^{max}$ number of nodes, i.e., $|\mathcal{I}^C| < N_\mathcal{I}^{max}$. Now, the modified interference power is:
   
    {\footnotesize
    \begin{equation}
    \label{eqn:pathloss1}
        P_{\mathcal{I}^C}(d)= \zeta. Q. \sum_{j \in \mathcal{I}^C} P_t d_{jX}^{-\eta} 10^{\frac{\psi}{10}} 
    \end{equation}
    }
    % where $\zeta=\max \{1, \frac{N_\mathcal{I}^{max}}{|\mathcal{I}^C|} \}$ and $|.|$ denotes the cardinality of a set.
\end{stp}

\begin{stp}
    We calculate the SIR value for each of the Interference Set Covers selected in Step~\ref{stp:1} in dB, as follows. 
    
    {\footnotesize
    \begin{equation}
    \label{eqn:pathloss2}
    % _{\mathcal{I}^C}
    SIR_X(d)=  10 \log_{10} \left( \frac{P_t d^{-\eta}10^{\frac{\psi }{10}}}{\zeta.\sum_{j \in \mathcal{I}^C} P_t d_{jX}^{-\eta}
    10^{\frac{\psi}{10}}} \right)
    \end{equation}
    }
\end{stp}

\subsection{Methodology for Selecting $d_{max}$}
\label{sec:select_d}
% \begin{itemize}
    % \item 
    
In order to properly select $d_{max}$, first of all, we need to estimate the distribution of the $SIR_X(d)$ using Eqn~\eqref{eqn:pathloss2},  which is not very straightforward as it involves division and summation of a large set of log normal random variables. The traditional log normal summation methods involve sampling and filtering to fit the distribution into an approximated log normal~\cite{safak1993statistical}. 
% Since, we are interested in the distance $d_{max}$ rather than the actual distribution of the interference power, 
We opt for similar approach where we collect a good number of samples, say $50000$, from each of the contributing log normal distributions, for a fixed $d$, to generate the SIR samples ($SIR_X(d)$) and use the SIR samples to determine the mean, $\mu_{SIR_X(d)}$, the variance of the SIR, $\sigma^2_{SIR_X(d)}$ and the empirical probability distribution function (PDF) of the ${SIR_X(d)}$. \todo{A rigorous mathematical PDF formulation is one of our future works.}
% Next, we calculate the outage probability for the separation distance $d$ as  $\Gamma (d)= \mathbb{P} (SIR_X(d)< SIR_{th})$. We repeat this procedure for a discrete set of values of $d \in [ 0,D_1]$ and choose the maximum allowable distance, $d_{max}$, to be the maximum value of $d$ such that $\Gamma(d)<\gamma$.
% $\argmax_{d \in \[0,D_1\]} \Gamma(d) $ values where  
% \todo{Now, if we choose the mean SIR i.e., $\mu_{{SIR(d)}}$, as our SIR estimate for the $d_{max}$ calculation, it will not give us a good bound as the probability of actual SIR being less than our estimate is $\sim 50\%$. To improve the reliability of our $SIR_X(d)$ estimation, thereafter, $d_{max}$ estimations, we choose the value $SIR_X(d)=\mu_{{SIR(d)}}-\sigma_{{SIR}(d)}$ as our SIR estimate. A more concrete approach will be to determine the empirical pdf of the ${SIR(d)}$ and choose $SIR_X(d)$ such that $\mathbb{P} \{SIR_X(d)>{SIR(d)}\}$ is very small. However, we take $SIR_X(d)=\mu_{{SIR(d)}}-\sigma_{{SIR}(d)}$ for simplicity as well as acceptable performance.}
% The reason behind subtracting the standard deviation from the samples is to guarantee a low probability for actual SIR being greater than  $SIR_X(d)$ (explained in the simulation results in Section~\ref{sec:simul}).} 
\emph{Note that in presence of fading, using simple path loss model, we can easily get the mean powers received from each interferer, which can be used to estimate $\frac{\mathbb{E} (P_r)}{\mathbb{E} (P_{\mathcal{I}})}$, but, not the mean SIR, i.e., $\mathbb{E}(SIR)= \mathbb{E} \left( \frac{P_r}{P_{\mathcal{I}}}\right) \not= \frac{\mathbb{E} (P_r)}{\mathbb{E} (P_{\mathcal{I}})} $.} 
% % \end{itemize}
% \begin{stp}
%     To properly select $d_{max}$, we first choose an acceptable value for $SIR_{th}$ and $\gamma$. Next, we use Eqn~\eqref{eqn:pathloss2} to estimate $SIR_X(d)=\mu_{SIR_{\mathcal{I}^C}(d)}-\sigma_{SIR_{\mathcal{I}^C}}(d)$ for a uniformly selected values of $d \in \left[ 0, D_1 \right] $. The highest value of $d$ that satisfies  $p(d)=\mathbb{P} (SIR_x(d)\geq SIR_{th})$, is the estimated $d_{max}$. 
% \end{stp}

\begin{stp}
    To properly select $d_{max}$, we first choose an acceptable value for $SIR_{th}$ and $\gamma$. Next, we use the samples of ${SIR}_X(d)$ to estimate the outage probability $\Gamma (d)=\mathbb{P} (SIR_X(d)< SIR_{th})$, for a uniformly selected values of $d \in \left[ 0, D_1 \right] $. The highest value of $d$ that satisfies $\Gamma (d) < \gamma$ is the estimated $d_{max}$. 
\end{stp}

\subsection{Orthogonal Code Bound For Interference Free Network}
% However, the network designer is free to choose any number of orthogonal codes in the system deployment. Thus we need to characterize the interference in terms of the number of codes employed, regardless of whether $N_{\mathcal{O}}< N_\mathcal{I}^{max}$ or not.
% It can be easily shown that if a set of $N_{\mathcal{O}}>1$ orthogonal codes are employed, the interference bound will decreases. 
First of all, say, $N_{\mathcal{O}}$ number of orthogonal codes are used and each node chooses a code randomly (all codes are equally likely to be chosen) and independently.
The new code specific interference power bound for a randomly selected Interference Set Cover ($\mathcal{I}^C$) will be:
{\footnotesize
\begin{equation}
\begin{split}
    P_{\mathcal{I}^C}(d|\mathcal{O}_T)&=\sum_{j=1}^{|\mathcal{I}^C|}P_{\mathcal{I}^C}^{j}\times \mathbbm{1}_{\{\mathcal{O}_j=\mathcal{O}_T\}}\\
    \mathbb{E} (P_{\mathcal{I}^C}(d))&=\frac{1}{(N_{\mathcal{O}})}\sum_{j=1}^{|\mathcal{I}^C|}\mathbb{E} (P_{\mathcal{I}^C}^{j})
\end{split}
\label{eqn:ortho_bound}
\end{equation}
}
where $\mathcal{O}_T$ is the code chosen by $T$, $P_{\mathcal{I}^C}^{j}$ denotes the interference power due to $j^{th}$ interferer in ${\mathcal{I}}^{C}$, and the indicator function $\mathbbm{1}_{\{\mathcal{O}_j=\mathcal{O}_T\}}$ denotes whether the $j^{th}$ interferer have chosen same code as the transmitter i.e., $\mathcal{O}_T$. 
% Now, assuming , the expected interference power becomes:
% {\footnotesize
% \begin{equation}
% \begin{split}
%     \mathbb{E} (P_{\mathcal{I}}^{total}(d))&=\frac{1}{(N_{\mathcal{O}})}\sum_{k=1}^{|\mathcal{I}^C|}\mathbb{E} (P_{\mathcal{I}}^{k})\\
%     % Var (P_{\mathcal{I}}^{total}(d))&=(N_{\mathcal{O}}-1) \sum_{k=1}^{|\mathcal{I}^C|} (P_{\mathcal{I}}^{k})^2
% \end{split}
% \label{eqn:ortho_bound1}
% \end{equation}
% }
\textbf{Notice that, the Interference Set Cover with maximum mean interference power will still give us the maximum mean estimated interference power in presence of orthogonal codes.}
\begin{stp}
    We use the estimated Interference Set Cover from Step~\ref{stp:1} to determine the new SIR bounds as follows.
    {\footnotesize
    \begin{equation}
    \label{eqn:ortho_sinr_bound}
    SIR_{\mathcal{I}^C}(d|\mathcal{O}_T)=  \frac{P_t d^{-\eta}10^{\frac{\psi}{10}}}{\zeta . \sum_{j \in \mathcal{I}^C}\left(P_t d_{jX}^{-\eta} 10^{\frac{\psi}{10}}\right). \mathbbm{1}_{\{\mathcal{O}_j=\mathcal{O}_T\}}}
    \end{equation}
    }
\end{stp}

Now, at any time instance, maximum $N^{max}=(N_{\mathcal{I}}^{max}+1)$ number of nodes can be active simultaneously. 
% Therefore, we need an orthogonal code alphabet of cardinality, $N_{\mathcal{O}}\geq N^{max}$, to guarantee interference free communication with very high probability ($\geq \kappa$). 
Given that $N_{\mathcal{O}}\geq N^{max}$, we deduce that (Proof in Appendix~\ref{App:AppendixO}):
{\footnotesize
\begin{equation}
    \begin{split}
        &\mathbb{P}(\mathbbm{1}_{\mathcal{I}0} = 1) \geq \prod_{i=1}^{N^{max}} \left( 1- \frac{i-1}{N_{\mathcal{O}}} \right)
        % &=\frac{N_{\mathcal{O}}!}{(N_{\mathcal{O}}- N^{max})! \times (N_{\mathcal{O}})^{N^{max}}}
    \end{split}
    \label{eqn:ortho}
\end{equation}
}
From Eqn~\eqref{eqn:ortho}, we can see that for  $N_{\mathcal{O}}\geq N^{max}$, $\prod_{i=1}^{N^{max}} \left( 1- \frac{i-1}{N_{\mathcal{O}}} \right)$ is a strictly increasing function of $N_{\mathcal{O}}$. 
\begin{stp}
    To find the optimum value of $N_{\mathcal{O}}$, we estimate $\prod_{i=1}^{N^{max}} \left( 1- \frac{i-1}{N_{\mathcal{O}}} \right)$ for increasing value of  $N_{\mathcal{O}}$ (starting from $N^{max}$), and select the minimum value of $N_{\mathcal{O}}$ such that $\prod_{i=1}^{N^{max}} \left( 1- \frac{i-1}{N_{\mathcal{O}}} \right) \geq \kappa$. 
    % The system designer can choose any value higher than that for interference free communication with high probability.  
\end{stp}

% Thus, we can estimate the distribution of the interference power as well as the SIR bound by same sampling techniques, as discussed earlier.

% Furthermore, we vary $N_{\mathcal{O}}$ and express the SIR as a function of the $d$ for each value of $N_{\mathcal{O}}$, in order to express SIR as a joint function ($f_3$) of $d$ and $N_{\mathcal{O}}$. 
% This mapping helps the system designer to properly choose the value of $N_{\mathcal{O}}$ and $d_{max}$.
% In this section, we first use the estimated $N_\mathcal{I}^{max}$ to find the minimum number of orthogonal codes ($N_{\mathcal{O}}$) such that $\mathbb{P}\{Zero Interference\}\geq \kappa$.
\section{Identification of Maximum Power Interference Set Cover}
\label{sec:dense_inter}
In the section, we identify the Interference Set Covers that result in the highest total interference power at a given receiver location, $X$, for both scenarios i.e., dense random network and robotic router network.

\subsection{Dense Random Network}
\label{sec:dense_rand}
% This is again similar to a circle packing problem but with relaxed overlapping condition, i.e., the circles may overlap as long as the centers are more than $R_1$ distance apart. 
% Moreover, the objective is now to maximize the interference power  
% A complimentary approach is a hexagon lattice structure as detailed in \cite{hekmat2004interference}. To find the Interference Set Cover that will result in the highest interference power at $X$ is a very hard (\todo{maybe NP-hard}) and computation intensive problem. 
In \cite{hekmat2004interference}, Hekmat and Van Mieghem showed that the mean interference power in CSMA Network is bounded by the interferers located along the hexagonal rings centred at the receiver's location, where the $i^{th}$ ring with each side length equal to $i\times D_1$ contains $6*i$ nodes. While the assumption of putting the receiver at the center is valid in the presence of RTS/CTS mechanism in CSMA, in reality, RTS/CTS mechanism is \textbf{NOT} employed in most of the enterprise wireless networks as well as Internet of Things (IoT) networks. \emph{In such cases, the transmitter is the node to be located at the center of the rings while the receiver is free to be located anywhere in the connected region of $T$. With this modification, the maximum feasible interference can actually be higher than the bound estimated in \cite{hekmat2004interference} e.g., when $X$ is located at the farthest point of the connected region of $T$. Moreover, for determining the number of nodes to deploy, we need to know the maximum separation distance ($d_{max}$) that can support an acceptable maximum interference level, in order to place a set of nodes in any area of deployment.} This requires us to modify the bounds to have a separation distance ($d$) dependency.
However, hexagonal packing is known to be the densest packing in circular spaces which leads us to believe that the distance dependent interference are also bounded by the interference power of the set of interferers located at hexagonal rings (similar to \cite{hekmat2004interference} but in an annular ring) around the Transmitter's location. 
With this assumption, our focus becomes restricted to all possible sets of locations that form such hexagonal packing. \emph{We can easily prove that, with the separation distance $d>0$, we only need to consider two different angular orientations of such hexagonal packing, as illustrated in Figures~\ref{fig:inter_cover} and~\ref{fig:inter_cover1}.} 
\begin{figure}[!ht]
    \centering
    \subfloat[Configuration 1]{\label{fig:inter_cover}\includegraphics[width=0.5\linewidth]{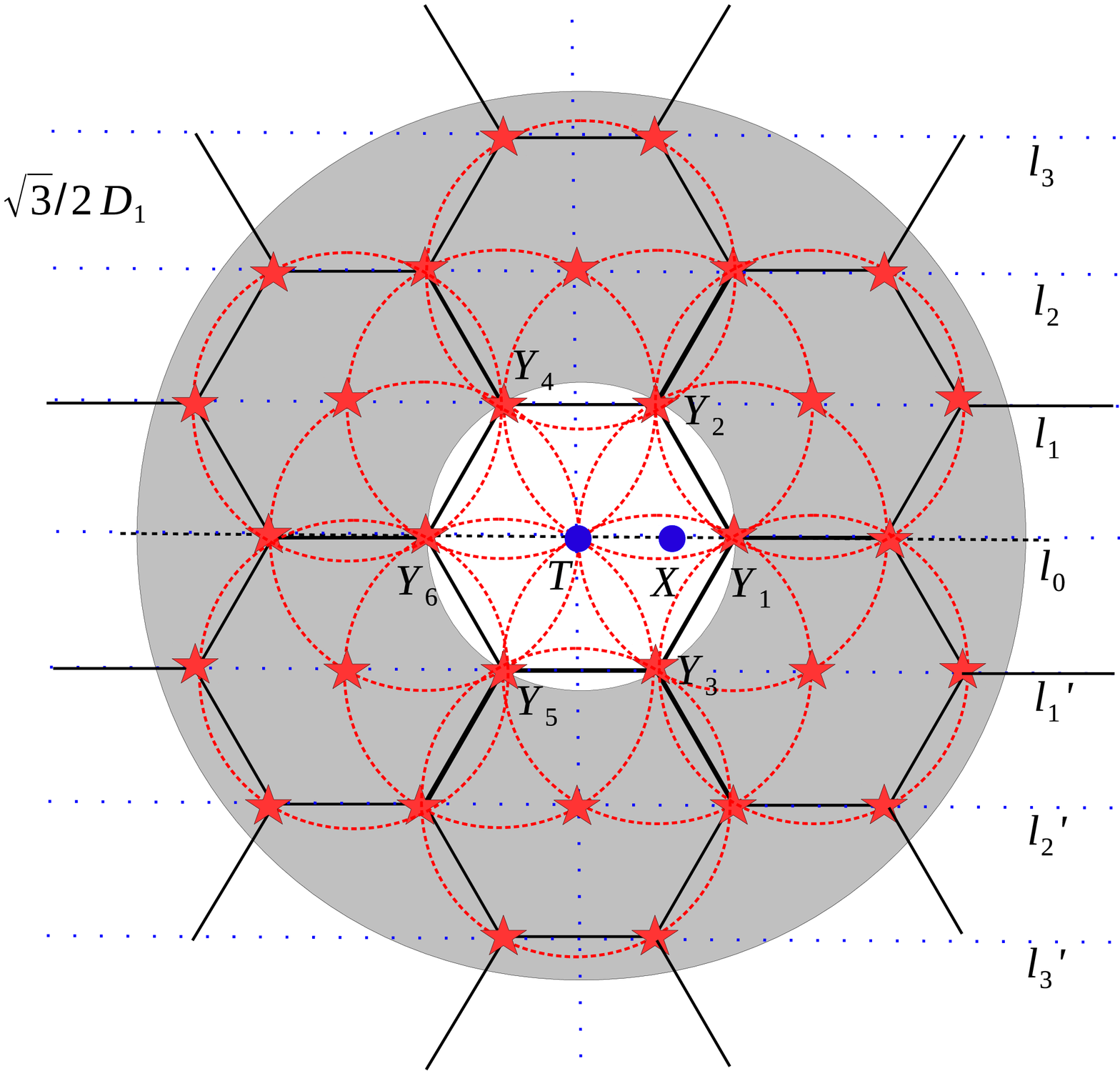}}
    \subfloat[Configuration 2]{\label{fig:inter_cover1}\includegraphics[width=0.5\linewidth]{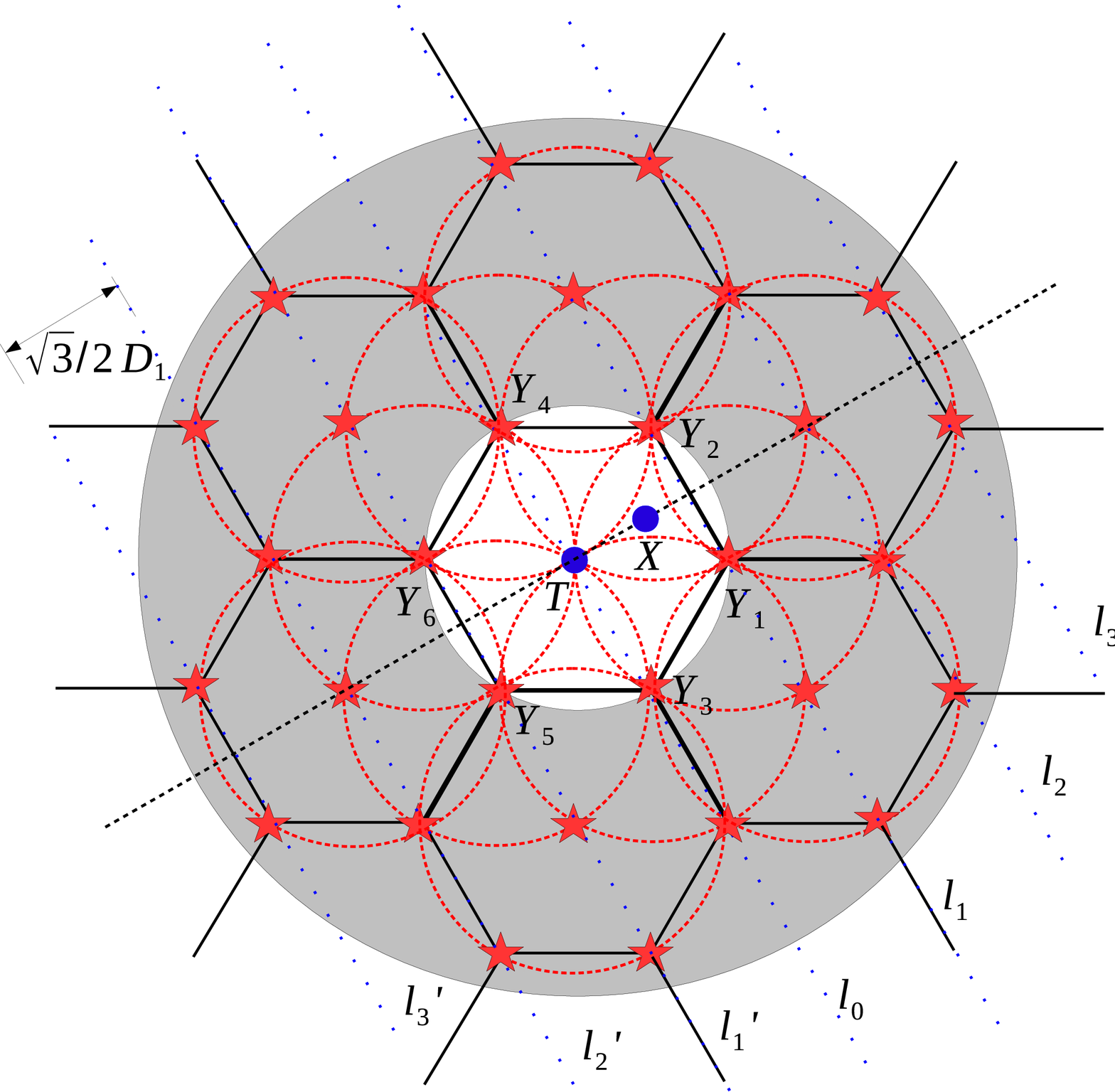} }
    \caption{Illustration of the Interference Set Covers For Estimation of Interference Upper Bound in a Dense Network}
\end{figure}

\begin{table*}[!ht]
    \caption{Interference Set Cover Node Locations for a Dense Network}
    \centering
    \begin{tabular}{|c|c|c|c|}
    \hline
    Line Number & \multirow{2}{*}{Configuration 1} & \multirow{2}{*}{Configuration 2}& \\
    (Illustrated in Figures\ref{fig:inter_cover} and~\ref{fig:inter_cover1})& & & \\
    \hline
    \multirow{2}{*}{$l_0$} & $\{(\pm jD_1,0)\}$ $\forall j \in \{1,2,\cdots, N_0+1\}$ & $\{(0,\pm jD_1)\}$ $\forall j \in \{1,2,\cdots, N_0+1\}$ & \multirow{ 2}{*}{$N_0=\lfloor \frac{D_2-D_1}{D_1} \rfloor $}\\
    &  &  & \\ \hline
    
    \multirow{2}{*}{$l_k$ where $k$ is odd } & $\{(\pm \frac{D_1(1+2\times j)}{2},  \frac{\sqrt{3}}{2}kD_1)\}$ & $\{( \frac{\sqrt{3}}{2}kD_1,\pm \frac{D_1(1+2\times j)}{2})\}$ & \multirow{ 3}{*}{$N_k=\lfloor \frac{\left(D_2^2-\frac{3}{4}k^2 D^2_1\right)^{\frac{1}{2}}}{D_1} \rfloor $}\\
     & & & \\ 
    $\forall k \in \{ 1, \lfloor \frac{2D_2}{\sqrt{3}D_1} \rfloor\}$ & $\forall j \in \{0,1,\cdots, N_k\}$ & $\forall j \in \{0,1,\cdots, N_k\}$ & \\ \hline
    
    \multirow{2}{*}{$l'_k$ where $k$ is odd } & $\{(\pm \frac{D_1(1+2\times j)}{2}, - \frac{\sqrt{3}}{2}kD_1)\}$ & $\{(- \frac{\sqrt{3}}{2}kD_1,\pm \frac{D_1(1+2\times j)}{2})\}$ & \multirow{ 3}{*}{$N_k=\lfloor \frac{\left(D_2^2-\frac{3}{4}k^2 D^2_1\right)^{\frac{1}{2}}}{D_1} \rfloor $}\\
     & & & \\ 
    $\forall k \in \{ 1, \lfloor \frac{2D_2}{\sqrt{3}D_1} \rfloor\}$ & $\forall j \in \{0,1,\cdots, N_k\}$ & $\forall j \in \{0,1,\cdots, N_k\}$ & \\ \hline
    
    \multirow{2}{*}{$l_k$ where $k$ is even } & $\{(\pm jD_1, \frac{\sqrt{3}}{2}kD_1)\}$ & $\{( \frac{\sqrt{3}}{2}kD_1,\pm jD_1\}$ & \multirow{ 2}{*}{$N_k=\lfloor \frac{\left(D_2^2-\frac{3}{4}k^2 D^2_1\right)^{\frac{1}{2}}}{D_1} \rfloor $}\\
    & & & \\ 
    $\forall k \in \{ 1, \lfloor \frac{2D_2}{\sqrt{3}D_1} \rfloor\}$ & $\forall j \in \{0,1,\cdots, N_k\}$ & $\forall j \in \{0,1,\cdots, N_k\}$ & \\ \hline

    \multirow{2}{*}{$l'_k$ where $k$ is even } & $\{(\pm jD_1, -\frac{\sqrt{3}}{2}kD_1)\}$ & $\{(-\frac{\sqrt{3}}{2}kD_1,\pm jD_1\}$ & \multirow{ 2}{*}{$N_k=\lfloor \frac{\left(D_2^2-\frac{3}{4}k^2 D^2_1\right)^{\frac{1}{2}}}{D_1} \rfloor $}\\
    & & & \\ 
    $\forall k \in \{ 1, \lfloor \frac{2D_2}{\sqrt{3}D_1} \rfloor\}$ & $\forall j \in \{0,1,\cdots, N_k\}$ & $\forall j \in \{0,1,\cdots, N_k\}$ & \\ \hline
    \end{tabular}
    \label{tab:isc_general}
\end{table*}

In the first type of configuration, which we refer to as \textbf{Configuration 1}, the closest interferer is located at the intersection of the inner boundary of the annulus and the line joining $T$ and $X$ (Illustrated in Figure~\ref{fig:inter_cover}). \emph{This configuration is generated by taking a greedy iterative approximation approach, where we start with an empty $\mathcal{I}^C$ and, in each iteration, we select a point on the annulus that is closest to the receiver $X$ and is not located in the connected regions of the nodes already added to $\mathcal{I}^C$.} In the second configuration, which we refer to as \textbf{Configuration 2} (illustrated in Figure~\ref{fig:inter_cover1}), the number of closest interferers is two and they are exactly $D_1$ distance apart from each other as well as from the transmitter. \emph{With this new initial condition, we can find the rest of the nodes, again, using the greedy approach.} Now, WLOG, we assume that $T$ is located at $(0,0)$ in a 2D domain, while $X$ is located at $(d,0)$. In this 2D domain, the positions of the interfering nodes for both of these Interference Set Covers are listed in Table~\ref{tab:isc_general}. It can be easily shown that these two configurations form the bound of the interference power for any configuration within same class i.e, with similar relative position between nodes with hexagonal corner positioning. %\todo{The proof of this claim is not presented in this paper for page limitations.}
Next, we calculate the interference and SIR for these two configurations according to Eqn.~\eqref{eqn:pathloss1} and~\eqref{eqn:pathloss2}. Then, we choose the maximum of these two interference estimates as our interference estimate, and minimum of these two SIR estimates as our SIR estimate. We perform this using the sampling method discussed in Section~\ref{sec:select_d}, where \textbf{we collect a large number of pairs of samples from these two configurations and take the highest interference power sample (or lowest SIR sample) from each pair as a sample for our estimated bounds. }

{
However, since this is an greedy solution, the resulting Interference Set Cover combination may not include the maximum number of interferer and, therefore, does not guarantee maximum possible interference power. \todo{Now say the greedy logic includes $n$ interferes. Then according to the greedy logic, it is most likely that the top $n$ interfering nodes of the maximum power Interference Set Cover will have less or equal interference power compared to the interference power from the greedily found Interference Set Cover.}  
To guarantee that our estimated interference power is no less than the maximum possible interference power, we multiply our estimated interference power by a correction factors, $\zeta=\max \{1, \frac{N_\mathcal{I}^{max}}{|\mathcal{I}^C|} \}$, where $N_{\mathcal{I}}^{max}$ denotes the maximum number of simultaneous interferers and $|.|$ denotes the cardinality of a set. The correction factor ($\zeta$) compensates for the cardinality of the Interference Set Cover i.e., if $|\mathcal{I}^C|<N_{\mathcal{I}}^{max}$.  \textbf{We found that the number of interferers estimated from the hexagonal packing is in fact also $N_{\mathcal{I}}^{max}$ for most of the cases.} Nonetheless, we can determine the maximum number of concurrent interfering nodes ($N_\mathcal{I}^{max}$) by formulating the problem as a circle packing problem \cite{hifi2009literature} as follows.

\begin{defn}
\textbf{Pack Problem:} Maximize the number of circles with radius $\left(\frac{D_1}{2}\right)$ that can be packed inside an annulus with inner and outer radius: $\left(\frac{D_1}{2}\right)$ and $\left(D_2+\frac{D_1}{2}\right)$, respectively.
\end{defn}
% \textbf{The goal of the circle packing problem is to maximize the number of circles with radius $\left(\frac{R_1}{2}\right)$ that can be packed inside an annulus with inner and outer radius: $\left(\frac{R_1}{2}\right)$ and $\left(R_2+\frac{R_1}{2}\right)$, respectively. We refer to this problem as \emph{Pack Problem}.}

\begin{lemma}
The cardinality of the solution to the Pack Problem is also the maximum cardinality of an Interference Set Cover. (Proof in Appendix~\ref{App:AppendixA}) \qed
\label{lemma:cardinality}
\end{lemma}
% \begin{proof}
% Refer to Appendix~\ref{App:AppendixA}
% \end{proof}

Note that, there exists a range of approximation solution to the circle packing problem~\cite{hifi2009literature}, which can be directly applied to solve this problem.
In this paper, we do not present any circle packing solution.
% in order to conform with the page restrictions.
Furthermore, since it is hard to analytically prove the correctness of our estimated bounds, we validate the bounds via a set of simulation experiments in Section~\ref{sec:simul}.}

% \vspace{-0.6cm}
\subsection{Interference Estimation for Robotic Router Network}
In this section, we focus on the interference estimation for our application specific context of robotic wireless network in a obstacle free environment.
% In order to put an lower bound on the number of robots TOrequired, we need to find a good estimate of the amount of interference a link can observe. 
% We take a generalized offline approach for estimating the maximum expected interference power at a certain receiver for a robotic router deployment scenario.
Before that, we make an assumption, based on two related works~\cite{williams2014route,yan2012robotic}, as follows.
\begin{assm}
\label{assm:straight}
{For a flow based robotic network in a obstacle free environment, if the goal is to optimize the flow performance in terms of SIR, the best configuration of robots allocated to that flow is to stay on the straight line joining the static endpoints.}
\end{assm}
\textbf{This assumption is justified by the work presented in~\cite{williams2014route} which shows that the best configuration of robots in order to optimize packet reception rate (which is directly related to SIR) of a flow based network is to evenly place them along the line segment joining the static endpoints. The work of Yan and Mostofi~\cite{yan2012robotic} further justify the linear arrangement of same flow nodes for Signal to Noise Ratio (SNR) based optimization goal.} 
In our analysis, we employ Assumption~\ref{assm:straight} to restrict the feasible positions of the interfering nodes, thereby, leading to better and tighter bounds on interference. In this context, we divide the interference into two components: Intra-flow interference and Inter flow interference.
These two components refer to the interference power from the nodes in the same flow as the transmitter $T$ and interference power from the nodes of different flows, respectively.
% If the distance between the end points of a flow,$i$ is $L_F(i)$, in the best case  we can place the robots at $d_{max}$ distance from each other in the straight line joining the endpoints of the flows, which will require $\frac{L_F(i)}{d_{max}}$ number of robots to satisfy flow $i$'s requirements. 
\begin{figure}[!ht]
    \centering
    \includegraphics[width=0.65\linewidth]{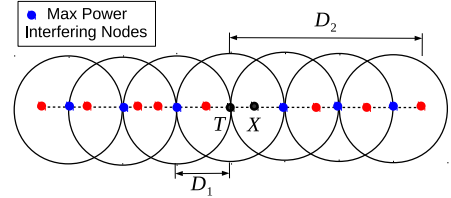} 
    \caption{Illustration of the Highest Power Intra-Flow Interference Set Cover}
    \label{fig:single_illus}
\end{figure}

\subsubsection{Intra-Flow Interference}
\label{sec:intraflow}

Our intra-flow interference estimation is based on the following lemma.
% For a specific transmitter ($T$) and receiver ($X$), only the nodes that are at a distance $D_1\leq d \leq D_2$ from $T$ are the potential interferers. However, because of CSMA, all the nodes cannot transmit at the same time. Therefore, summing up the interference power of all such nodes is an overkill.

\begin{lemma}
{The maximum expected Intra-flow interference power for a link corresponds to the sum of interference powers from nodes located at distances $\{D_1,2.D_1,\cdots k.D_1\}$ from the transmitter node $T$ along the line segment joining the flow endpoints, where $k.D_1\leq D_2$.} (Proof in Appendix~\ref{App:AppendixB}) \qed
\label{lemma:2}
\end{lemma}
% \begin{proof}
% Please Refer to Appendix~\ref{App:AppendixB}.
% \end{proof}

Therefore, the maximum number of intra-flow interferers is $2 \left( \lfloor \frac{D_2-D_1}{D_1} \rfloor +1 \right)$, where the factor $2$ accounts for both sides.
In Figure~\ref{fig:single_illus}, we present an illustration of such scenario. Thus, the set of nodes that will result in the highest intra-flow interference power are located at $\{ (\pm jD_1,0)\} \ \ \forall j\in \{1, \cdots, \lfloor \frac{D_2-D_1}{D_1} \rfloor+1 \}$ in the 2-dimensional area of interest. Interestingly, these set of locations are same as the line $l_0$ of \textbf{Configuration 1} discussed in Section~\ref{sec:dense_rand}.

%\IEEEpubidadjcol

\subsubsection{Inter-Flow Interference}
\label{sec:inter-flow}
In realistic scenarios, there will be more than one flows in the network where robots assigned to different flows can interfere as well. We refer to such interference as the \emph{Inter-flow interference}. Now, the interferers can be located in the annular transition region around the transmitter, while the nodes allocated to same flow stay on the straight line joining the endpoints of the respective flow (according to Assumption~\ref{assm:straight}).
% . Now finding the worst case inter-flow interference is a hard problem. In order to simplify the situation, we still use the fact that the optimal positioning of robots for each flow is to place them in the line joining the respective source and sink of the flow. 
In this section, we start the bound estimation with a two flow network, followed by a network with $M$ flows. In this context, we make a key assumption about the maximum power Interference Set Cover for multi-flow scenario, as follows.

% \begin{proof}
% App:AppendixB
% \end{proof}
\begin{figure}[!ht]
    \centering
    \subfloat[Two flow Case]{\label{fig:new_2_flow_cover}\includegraphics[width=0.5\linewidth]{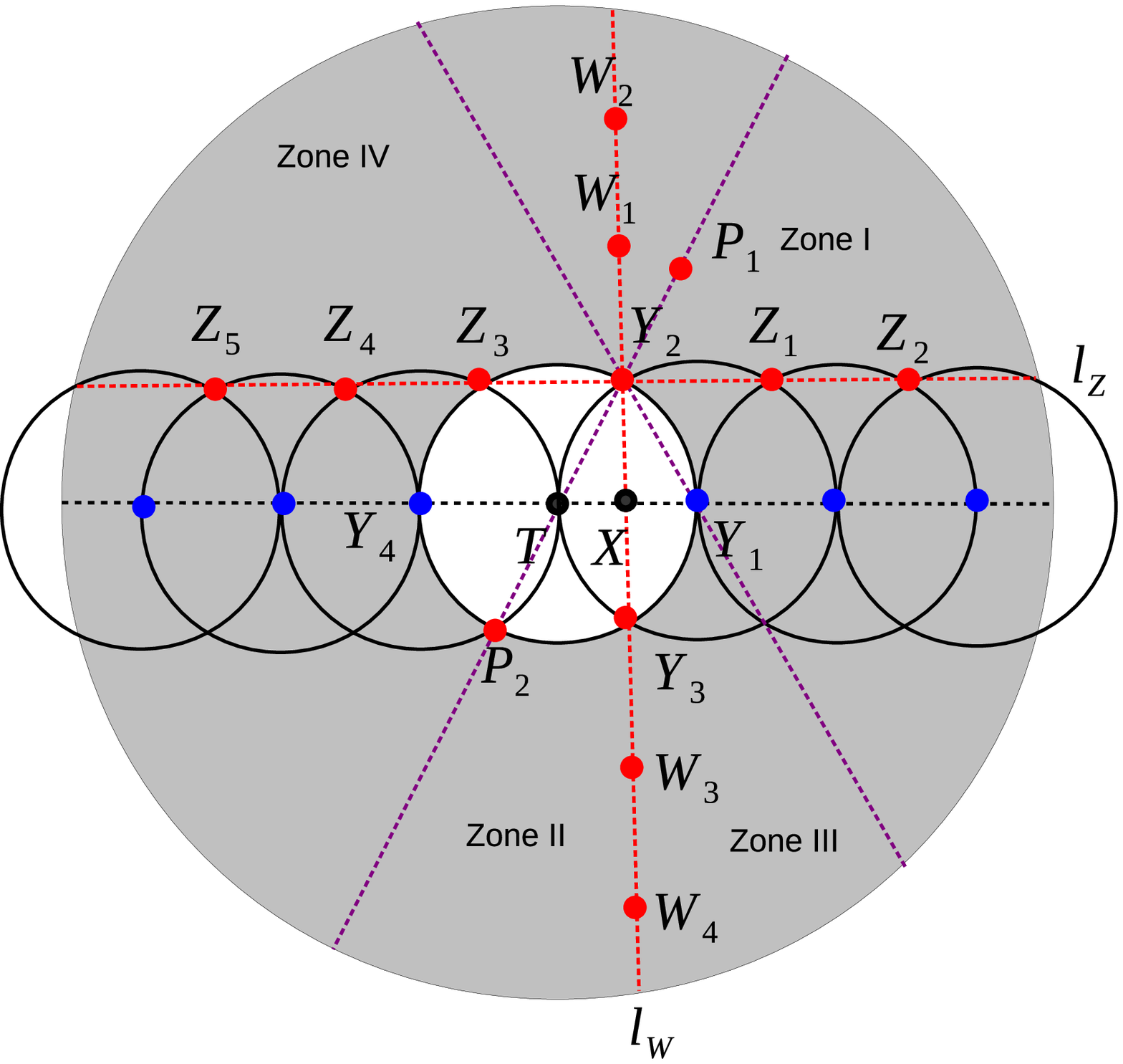} }
    \subfloat[$M$ Flow Case]{\label{fig:new_multi_flow_cover}\includegraphics[width=0.5\linewidth]{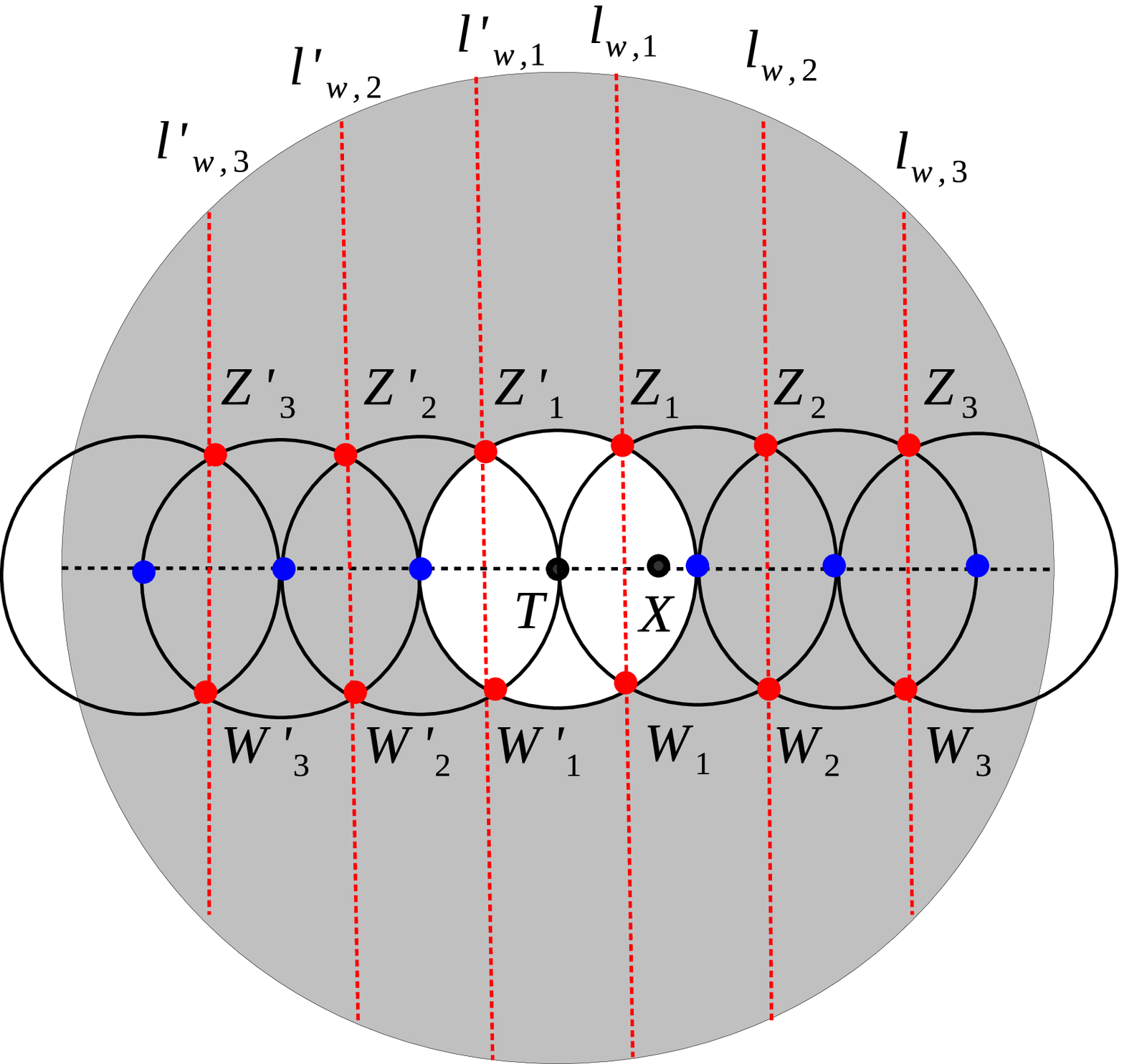} } 
    \caption{Illustration of the Multi Flow Interference Estimation (Blue Nodes: Intra-Flow Interferer, Red Nodes: Inter-Flow Interferer)}
\end{figure}
\begin{assm}
{For any transmitter-receiver node pair of a flow, the intra-flow maximum power Interference Set Cover estimated in section~\ref{sec:intraflow} is always part of the maximum power Interference Set Cover in presence of multiple flows.} \qed
\label{assm:1}
\end{assm}
% \begin{proof}[Justification]
\emph{The reason behind this assumption is mainly the fact that
% the greedy approach mentioned in the Section~\ref{sec:dense_rand}, where we assume that the node closest to the receiver $X$ is always part of the maximum power Interference Set Cover. 
in practical deployment, some node-pairs might not have any inter-flow interference at all (e.g., single flow network). Therefore, neglecting any of the intra-flow interfering nodes will lead to a incorrect estimate of the interference in such cases.} 
% Since both these objectives are fulfilled by this assumption, we believe this is the right approach. 
% \end{proof}
Under the given assumption, our next step is to find another line segment that will generate the maximum inter-flow interference power, for two flow cases. In general case with $M$ flows, we need to find $M-1$ other line segments such that carefully placed set of interferers on those segments result in the highest inter-flow interference power.  Now, following the greedy approach mentioned in the Section~\ref{sec:dense_rand}, the second flow should contain $Y_2$ or $Y_3$ or both, in Figure~\ref{fig:new_2_flow_cover}, since they are the next closest points to $X$ after the Intra-flow interference set cover nodes are accounted for. 
% the Before moving on, we introduce another lemma related to our estimation process.

\begin{lemma}
{Among the possible line segments through $Y_2$ or $Y_3$ or both, we just need to consider $l_{\mathcal{Z}}$ and $l_{\mathcal{W}}$ in Figure~\ref{fig:new_2_flow_cover} for estimating the bound on the interference power for two flow case.} (Proof in Appendix~\ref{App:AppendixD}) \qed
\label{lemma:lines_choice}
\end{lemma}
% \begin{proof}

% \end{proof}

The set of nodes on $l_W$ that will result in highest interference power should be located at $(\frac{D_1}{2}, \pm (\frac{\sqrt{3}}{2}D_1+jD_1))\}$ $\forall j \in \{0,1,\cdots, \lfloor \frac{\left(D_2^2-\frac{D_1^2}{4}\right)^{\frac{1}{2}}- \frac{\sqrt{3}}{2}D_1}{D_1} \rfloor \}$. On the other hand, The maximum power interference set cover node locations on $l_Z$ are same as the line $l_1$ of \textbf{Configuration 1}, listed in to Table~\ref{tab:isc_general}. 
Now, the inter flow interference power is $\max \{ P_{\mathcal{I}}^{l_\mathcal{W}}(d), P_{\mathcal{I}}^{l_{\mathcal{Z}}}(d) \}$, where $P_{\mathcal{I}}^{l_{\mathcal{W}}}$ and $ P_{\mathcal{I}}^{l_{\mathcal{Z}}}$ denotes the total maximum interference power for nodes in line $l_{\mathcal{W}}$ and $l_{\mathcal{Z}}$, respectively.

Next, we extend this concept to $M$ flow scenario i.e., maximum $M-1$ interfering flows. 
For a fixed pair of transmitter and receiver node of a flow with $M-1$ interfering flows, we need to consider two class of configurations. The mean inter-flow interference power bound of the \textbf{first class of configurations} is calculated by summing up the total interference power of the first $M'$ lines from the set $\{l_1,l'_1,l_2,l'_2,\cdots,l_K,l'_K\}$ in Figure~\ref{fig:inter_cover}, where $K=\lfloor \frac{2D_2}{\sqrt{3}D_1} \rfloor$ and $M'=\min \{M-1,2K\}$.
    Now, for the bound estimation of \textbf{second class of configurations}, we consider the line segment joining the closest pair of nodes at any point of time. 
More precisely, we choose $M'$ pairs of nodes from the pairs illustrated in Figure~\ref{fig:new_multi_flow_cover} as $\{ (Z_1,W_1),(Z_2,W_2),(Z'_1,W'_1),(Z_3,W_3), \cdots, (Z'_K,W'_K)\}$ where $K=\left(\lfloor \frac{(D_2-D_1/2)}{D_1} \rfloor+1 \right)$, $M'=\min \{M-1,2K\}$, and the pairs are sorted in terms of the respective distances to the receiver. 
Thus, the flows situated along lines $l_{\mathcal{W}, i}$ and  $l'_{\mathcal{W}, i}$ , $i \in \{1,2,\cdots,K\}$ determine the second type of interference bound in our estimation. The respective locations of the interferers are illustrated in Table~\ref{tab:isc-flow}.
Next, we compare these two bounds and take the maximum of them as the estimated interference power bound. We prove the validity of this bound through a set of MATLAB based simulation experiments, discussed in Section~\ref{sec:simul}.

\begin{table}[t]
    \centering
    \caption{Interference Set Cover Node Locations for a Flow Based  Network}
    \begin{tabular}{|c|c|}
    \hline
    Line Number & \multirow{2}{*}{} \\
    (Illustrated in Figures~\ref{fig:new_multi_flow_cover})& \\
    \hline
    \multirow{2}{*}{$l_{W,k}$} & $\{((2k+1)\frac{D_1}{2}, \pm (\frac{\sqrt{3}}{2}D_1+jD_1))\}$   \\
     & \\ 
    $\forall k \in \{ 0, \lfloor \frac{(D_2-D_1)}{2D_1} \rfloor\}$ & $\forall j \in \{0,1,\cdots, N_{W,k}\}$   \\ \hline
    \multirow{2}{*}{$l'_{W,k}$} & $\{ (-(2k+1)\frac{D_1}{2}, \pm (\frac{\sqrt{3}}{2}D_1+jD_1))\}$ \\
     & \\ 
    $\forall k \in \{ 0, \lfloor \frac{(D_2-D_1)}{2D_1} \rfloor\}$ & $\forall j \in \{0,1,\cdots, N_{W,k}\}$  \\ \hline
    
    \multicolumn{2}{|c|}{\multirow{ 3}{*}{$N_{W,k}=\lfloor \frac{\left(D_2^2-\frac{\{(2k+1)D_1\}^2}{4}\right)^{\frac{1}{2}}- \frac{\sqrt{3}}{2}D_1}{D_1} \rfloor $}} \\ 
    \multicolumn{2}{|c|}{}\\ 
    \multicolumn{2}{|c|}{}\\ 
    % \multicolumn{2}{|c|}{}\\ 
    % \multicolumn{2}{|c|}{where $i \in \{1, \cdots, \lfloor \frac{D_2-D_1}{D_1} \rfloor+1\} $} \\
    % \multicolumn{2}{|c|}{}\\ 
    \hline
    \end{tabular}
    \label{tab:isc-flow}
\end{table}
\begin{figure*}[!ht]
\centering
\subfloat[]{\label{fig:SIR_Compare_deter} \includegraphics[width=0.33\linewidth, height=0.25\linewidth]{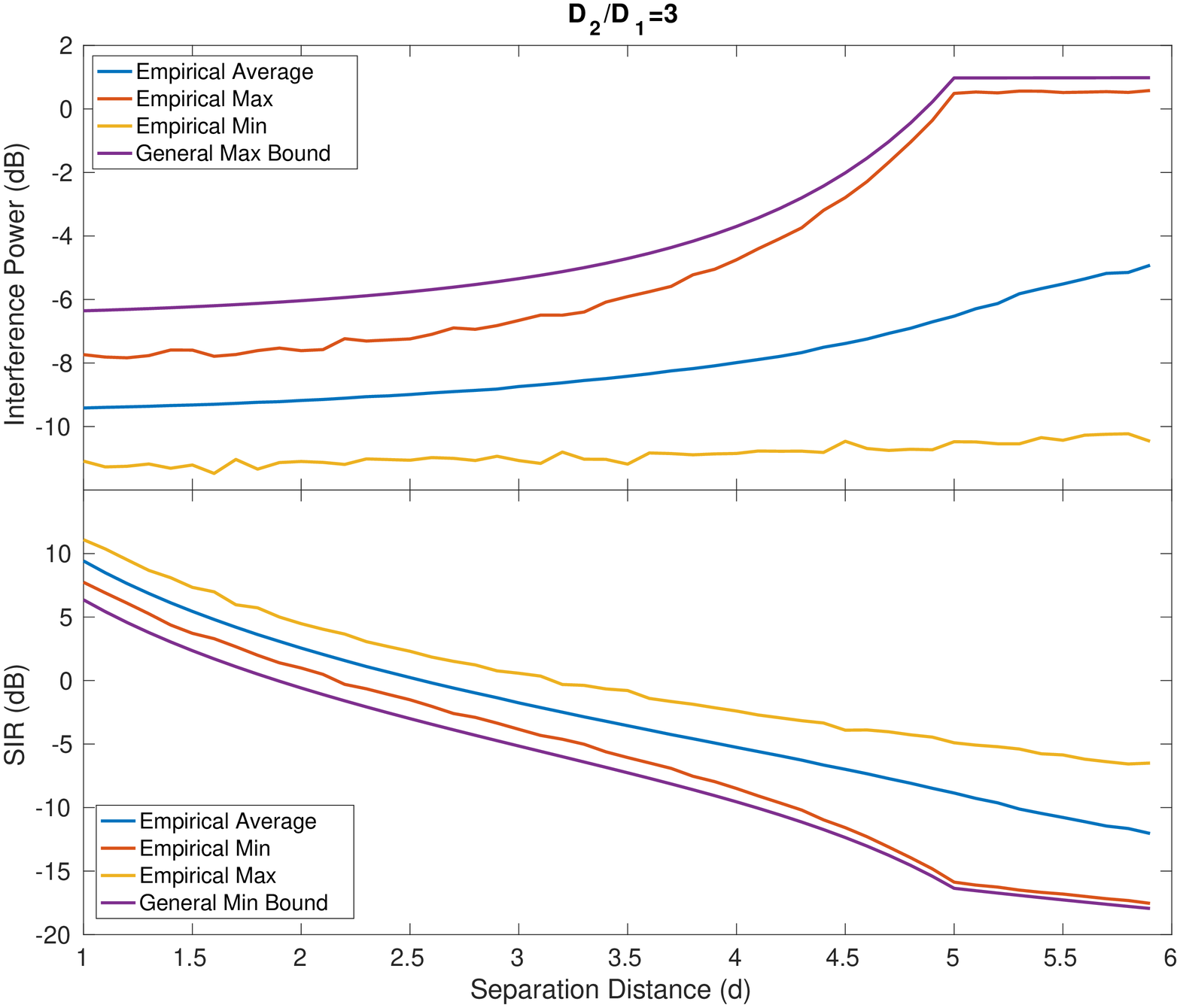}}
\subfloat[]{\label{fig:SIR_Compare_stoc} \includegraphics[width=0.33\linewidth, height=0.25\linewidth]{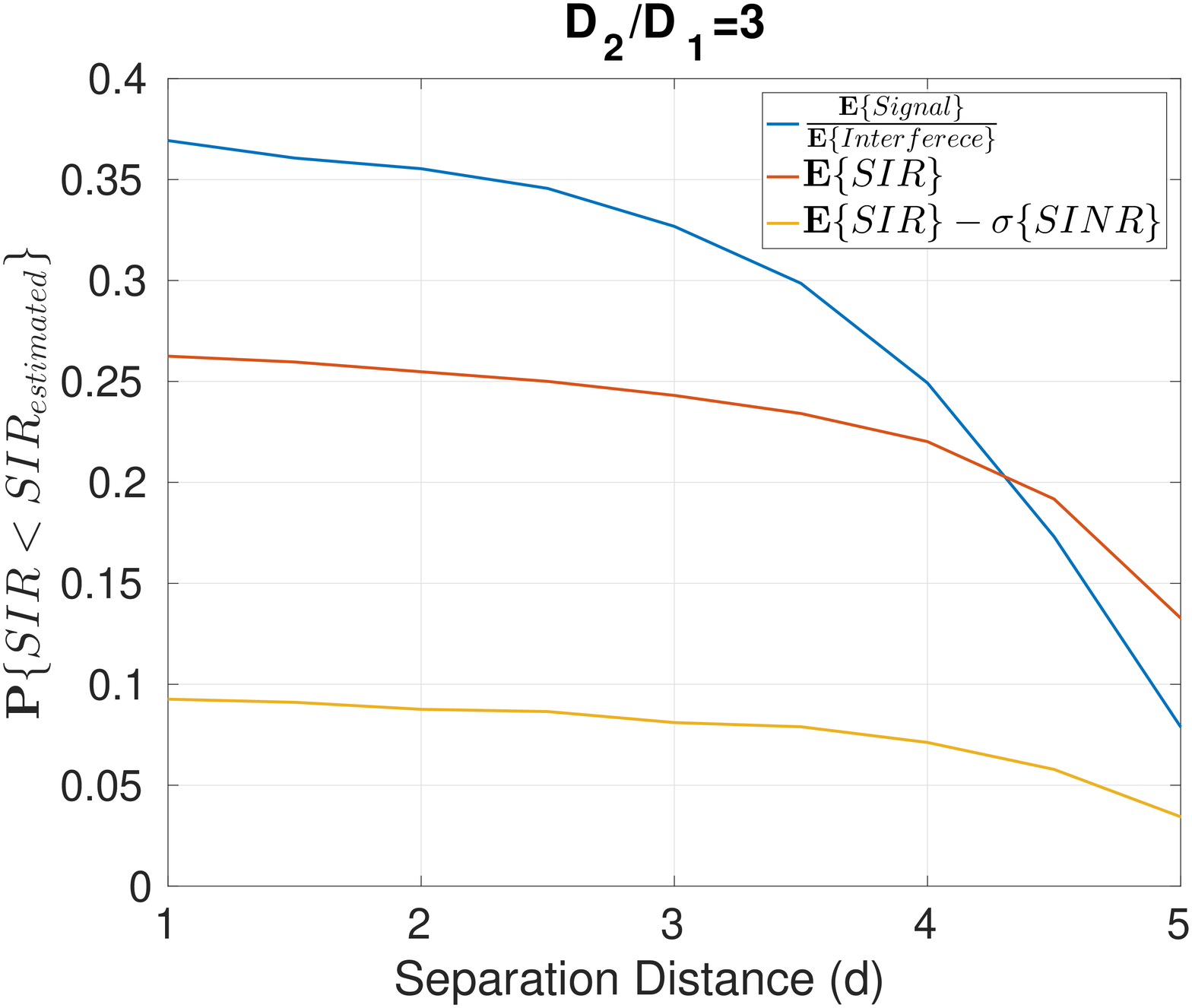}}
\subfloat[]{\label{fig:SIR_Compare_deter_ortho} \includegraphics[width=0.33\linewidth, height=0.25\linewidth]{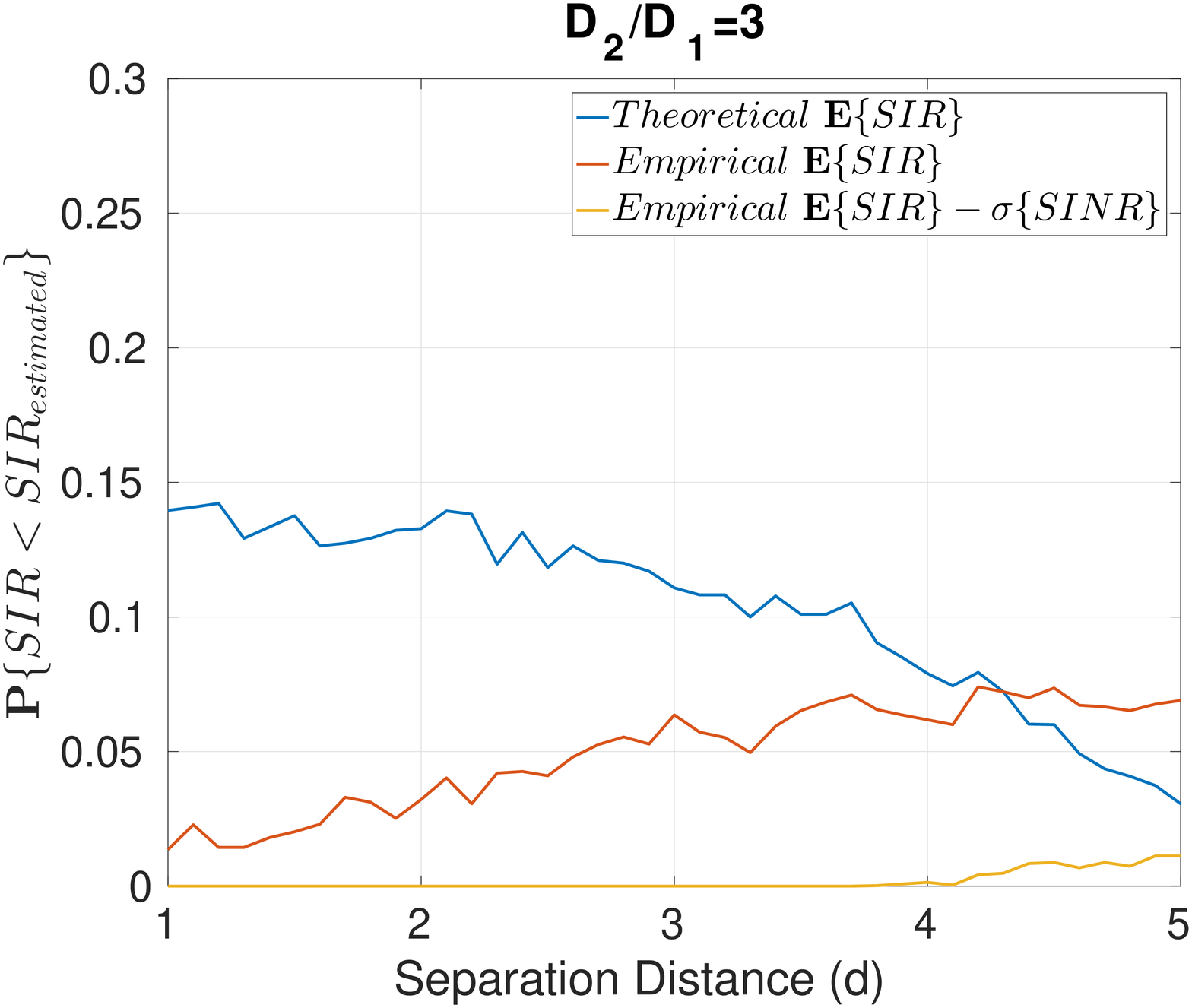}}

\caption{(a) Validation of Estimated Interference Power (Top) and SIR (Bottom) Bounds in dB, for Dense Network with No fading (b) Probability that Actual SIR is Lower than the estimated Minimum SIR with Log-Normal Fading with variance $\sigma^2=4$ (c) Probability that Actual SIR is Lower than the estimated Minimum SIR with NO Fading but in Presence of 10 Orthogonal Codes}
\end{figure*}

\section{Simulation Results}
\label{sec:simul}
In this section, we verify our proposed $d$ dependent bounds on the interference and SIR, through a set of MATLAB 8.1 based experiments performed on a machine with 3.40 GHz Intel i7 processor and 12GB RAM. For this set of experiments, we fix the values of the transmitter powers and the path loss exponent at $P_t=1$ and $\eta=2.2$, respectively. The value of $\eta=2.2$ is motivated by our experiences from real outdoor experiments (from a different project).
% , we,  nonetheless, verify our results for a range of different values of $\eta \in [1.5 6]$. 
% However, we do not present those results in this paper to avoid redundancy as well as to meet page requirements. 
As a measure of the annular transition region area, we choose the ratio of $\frac{D_2}{D_1}=\{3,6\}$ as the typical RSSI CCA thresholds are separated by $10$dB to $15$dB~\cite{zeng2014first}. The absolute value of $D_1$ is randomly selected to be $6m$ as \emph{the major factors that controls the performance is the $\frac{D_2}{D_1}$ ratio, not the absolute values of $D_1$ and $D_2$.} 
% We verified this fact though our experiments as well, nonetheless, not present in this paper due to page restrictions. 
With these initializations, we vary the separation distance $d$ from $1m$ to $D_1-1m$ with granularity of $0.1m$ to plot the separation distance dependent bounds.
\

% \begin{figure}[!ht]
% \centering
% \includegraphics[width=0.8\linewidth]{./new_simulations/deterministic_3.eps}
% \caption{Comparison of Estimated Interference Power (Top) and SIR (Bottom) in dB, for Dense Network with No fading}
% \label{fig:SIR_Compare_deter}
% \end{figure}

% \begin{figure}[!ht]
% \centering
% { \includegraphics[width=0.8\linewidth, height=0.5\linewidth]{./new_simulations/stoc_3.eps}}
% \caption{Probability that Actual SIR is Lower than the estimated Minimum SIR for Dense Network with Log-Normal Fading with variance $\sigma^2=4$.}
% \label{fig:SIR_Compare_stoc}
% \end{figure}
\begin{algorithm}[t]
\begin{algorithmic}[1]
\Procedure{Generate}{ }
\State Initialize a Dense Set of Nodes: $\mathcal{I}^{D}$
\State Initialize $\mathcal{I}^{S}$ as a empty set
\While{$\mathcal{I}^{D}$ is not Empty}
\State Randomly select $v\in \mathcal{I}^{D}$
\State $\mathcal{I}^{S}=\mathcal{I}^{S}\cup v$
\State $\mathcal{B}_v=\{i | i\in\mathcal{I}^{D}\  \& \ d_{iv}<D_1\}$
\State $\mathcal{I}^{D}=\mathcal{I}^{D} \setminus \mathcal{B}_v$
\EndWhile
\EndProcedure
\end{algorithmic}
\caption{Generate a random set of Interferer}
\label{Alg-generate}
\end{algorithm}

First, we verify the bounds for a general dense network, where the interfering nodes are uniformly distributed over the annular transition region around $T$.
% th inner and outer radius as $D_1$ and $D_2$, respectively. 
To verify the bounds, we randomly generate $1000$ sets of interfering nodes, for a fixed value of $d$, using Algorithm~\ref{Alg-generate}.
% to generate and add it to a list, say $\mathcal{I}^{S}$ (Initialized as empty list), and remove all the nodes within $D_1$ distance of the selected node (including the itself). Next, we randomly select another interfering node from the remaining interfering nodes and again remove the nodes within $D_1$ distance from the newly selected node. We repeat this procedure until no interfering node is left to be added to $\mathcal{I}^{S}$. Then, we sum up the interference powers for the set of selected nodes, say $\mathcal{I}^{S}$, for a fixed value of $d$.
% We repeat this procedure $1000$ times for each value of $d$. 
In Figure~\ref{fig:SIR_Compare_deter}, we compare our estimated interference power and estimated SIR, with the interference powers and SIR of the generated $\mathcal{I}^{S}$ sets, for no fading scenario and $\frac{D_2}{D_1}=3$. Figure~\ref{fig:SIR_Compare_deter} clearly validates our $d$ dependent interference and SIR bounds for a general dense network in absence of fading. Next, we perform similar experiments but in the presence of log normal fading of variance $\sigma^2=4$ and $\frac{D_2}{D_1}=3$. In this set of experiments, the estimated bounds for each value of $d$ are some probability distributions, rather than deterministic values. In this context, we empirically collect a set of $50000$ samples ($SIR(d)$) from the distributions estimated according to Eqn~\eqref{eqn:pathloss2} and estimate the mean, $\mu_{SIR_X(d)}$ and the variance of the SIR, $\sigma^2_{{SIR}_X(d)}$. Next, we collect $50000$ sample from each generated $\mathcal{I}^{S}$ and empirically compute the probabilities, $\mathbb{P} (SIR_{\mathcal{I}^{S}} < \mu_{SIR_X(d)})$ , $\mathbb{P} (SIR_{\mathcal{I}^{S}} < \mu_{SIR_X(d)} -\sigma_{{SIR}_X(d)})$ and $\mathbb{P} (SIR_{\mathcal{I}^{S}} < \frac{\mathbb{E} (Signal)}{\mathbb{E} (Interference})$. We plot the results in Figure~\ref{fig:SIR_Compare_stoc} which shows that the estimated SIR mean (from Eqn~\eqref{eqn:pathloss2}) is higher than the actual SIR for around $25\%$ of the cases, while $\mu_{SIR_X(d)}-\sigma_{{SIR_X}(d)}$ is higher than the actual SIR for only $10\%$ of the case. Thus, if we were to choose a deterministic value for the bound rather than a distribution, $\mu_{SIR_X(d)}-\sigma_{{SIR_X}(d)}$ is considered as a good estimate. Next, we use similar sampling method to generate the orthogonal code based SIR bounds when the number of codes used is $10$, while the maximum number of simultaneously interfering node is $38$ (For $D_2/D_1=3$). In this set of experiments, each node randomly selects a code from the code alphabet. But, we only sum up the interference powers of the interferers that select the same code as the transmitter. We apply the same method for each of the $\mathcal{I}^{S}$ set as well to validate our bounds and plot the probabilities $\mathbb{P} (SIR_{\mathcal{I}^{S}} < \mu_{SIR_X(d)})$ and $\mathbb{P} (SIR_{\mathcal{I}^{S}} < \mu_{SIR_X(d)} -\sigma_{{SIR}_X(d)})$ in Figure~\ref{fig:SIR_Compare_deter_ortho}, for log normal fading scenario. Figure~\ref{fig:SIR_Compare_deter_ortho} shows that our proposed bound also works well in presence of orthogonal codes.
% P^{\mathcal{O}}_1=\mathbb{P} (SIR_{\mathcal{I}^{S}} < \mu_{{SIR(d)}})$, $P^{\mathcal{O}}_2=\mathbb{P} (SIR_{\mathcal{I}^{S}} < \mu_{SIR(d)} -\sigma_{{SIR}(d)})$ 

% Thus, we select the mean minus 2 times variance to be our bound estimate in presence of fading.
\begin{figure*}[!ht]
\centering
\subfloat[]{\label{fig:SIR_Compare_flow} \includegraphics[width=0.33\linewidth, height=0.25\linewidth]{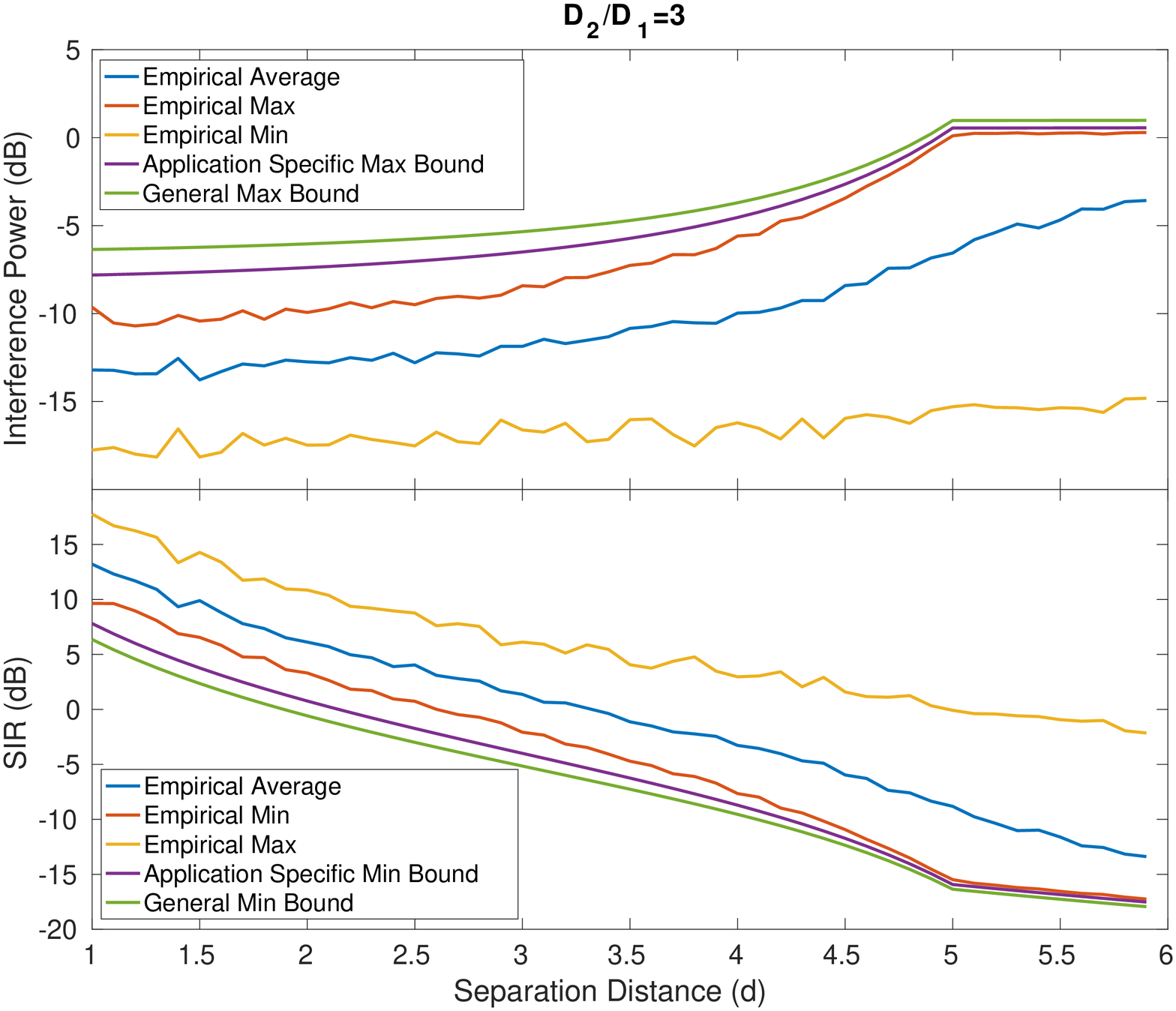}}
\subfloat[]{\label{fig:SIR_Compare_num} \includegraphics[width=0.33\linewidth, height=0.25\linewidth]{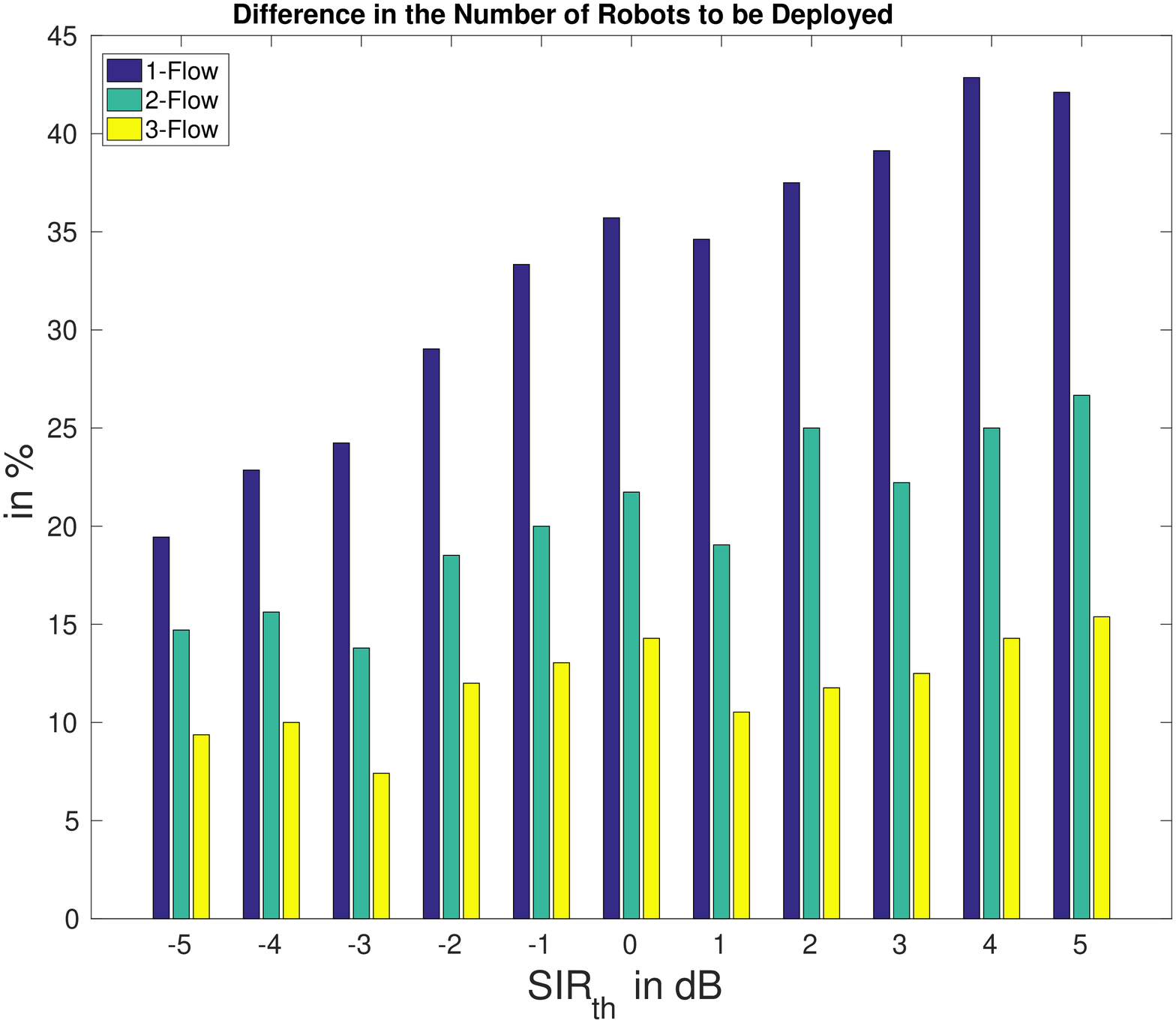}}
\subfloat[]{\label{fig:SIR_Compare_flow_stoc} \includegraphics[width=0.33\linewidth, height=0.25\linewidth]{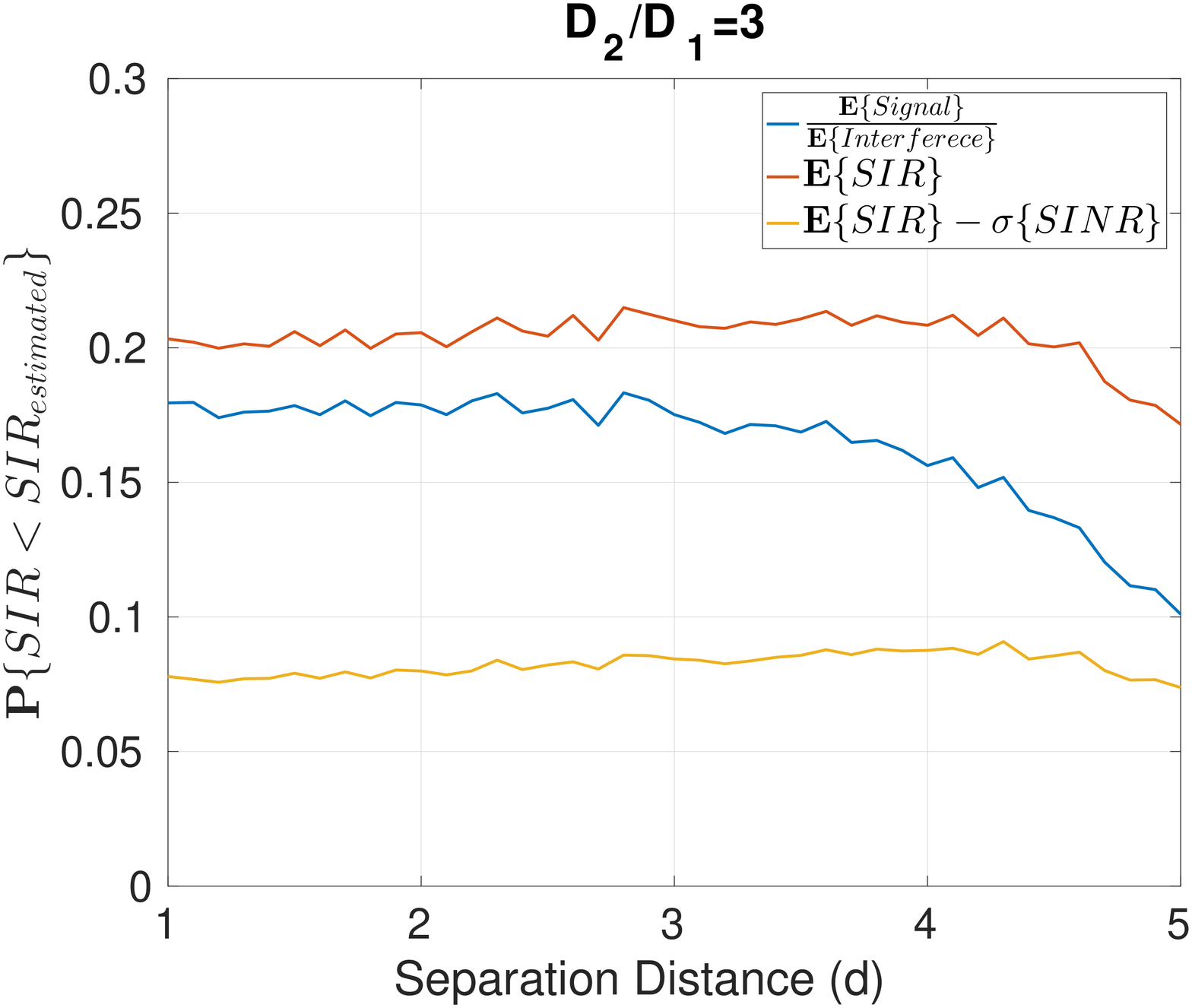}}
\caption{For a 3 Flow Network: (a) Validation of Estimated Interference Bound (Top) and SIR Bound (bottom) with No Fading (b) Illustration of Less Number of Robots to be Deployed with Our Application Specific Bound with No Fading (c) Probability that Actual SIR is Lower than the estimated Minimum SIR with Log-Normal Fading with variance $\sigma^2=4$}
\end{figure*}

% \begin{figure}[!ht]
% \centering
% \includegraphics[width=\linewidth]{./new_simulations/stoc_3.eps}
% \caption{Comparison of Estimated Interference Power}
% \label{fig:SIR_Compare_stoc}
% \end{figure}s

Similar to the generic dense wireless network, we perform a set of bound tests for the robotic network scenario for $\frac{D_2}{D_1}=3$. In this case, we randomly select two pairs of endpoints (i.e., we consider a 3 flow network) along the circumference of the outer circle with radius $D_2$, which are the flow endpoint for two other flows. Next, we place a dense set of points along each of the randomly selected flow segments as well as the line segment joining the transmitter $T$ and the receiver $X$ to include the intra-flow interference. Then, we use  Algorithm~\ref{Alg-generate} to generate $1000$ sets of interfering nodes for each value of $d$ and for each of the $500$ randomly generated sets of flow endpoints. In all cases, the total interference power is bounded by our proposed theoretical maximum interference power, for no fading scenario, as illustrated in Figure~\ref{fig:SIR_Compare_flow}. This figure also shows that our application specific bounds are much tighter than the generic bound. In order to illustrate the impact of this improvement, we also plot the difference in the number of robots required to cover a distance of $100m$ for different values of $SIR_{th}\in [-5dB,5dB]$ in Figure~\ref{fig:SIR_Compare_num} for $\frac{D_2}{D_1}=\{3\}$. Figure~\ref{fig:SIR_Compare_num} clearly illustrates that with our improved bound, the required number of robots to guarantee some target SIR requirements, is significantly lower than the generic bound based number of robots estimations, ranging from a maximum of $\sim 45\%$ for single flow network to a minimum of  $\sim 10\%$ for a three flow network. The improvement is significant for less number of flows, as for higher number of flows ($\sim 6-7$ flows) the general dense network bound becomes dominant, which is quite intuitive.  Next, similar to the generic bound, in Figure~\ref{fig:SIR_Compare_flow_stoc} we compare the bounds in presence of fading to show that the estimated $\mu_{SIR_X(d)}-\sigma_{{SIR_X}(d)}$ is higher than the actual SIR for only $10\%$ of the case, for $\frac{D_2}{D_1}=3$.

% \begin{figure}[!ht]
% \centering
% {\label{fig:SIR_Compare_flow} \includegraphics[width=0.8\linewidth]{./new_simulations/deterministic_flow_3.eps}}
% \caption{Comparison of Estimated SIR}
% \end{figure}

% \begin{figure}[!ht]
% \centering
% {\label{fig:SIR_Compare_flow_stoc} \includegraphics[width=0.8\linewidth]{./new_simulations/stoc_flow_3.eps}}
% \caption{Comparison of Estimated SIR}
% \end{figure}

%  \begin{figure}
%      \centering
%      \includegraphics[width=0.8\linewidth]{./new_simulations/robo_diff.eps}
%      \caption{Comparison of Estimated SIR}
%      \label{fig:SIR_Compare_num}
%  \end{figure}

\section{Conclusion}

In this paper, we proposed a method for estimation of the maximum interference and minimum achievable SIR for a link of length $d$ in an unknown environment while CSMA-CA or equivalent MAC layer protocols are employed. First, we demonstrate a strong dependency of these bounds on the transmitter-receiver separation distance $d$. Next, by considering two different scenarios: generic dense network and robotic router network; we demonstrate that we can formulate better and tighter bounds by exploiting the network topology structure which infact improves our main goal of estimating the number of nodes to be deployed for our robotic router network in order to guarantee some network performance.
% . For this purpose, we use the application of robotic router to form communication relays between a set of communication endpoint pairs where the bound is particularly useful for determining the number of relays to be deployed in order to satisfy the performance constraints. 
We also perform a set of MATLAB based simulation results that validate our findings. This work is a part of our bigger project of development of a CSMA Aware Autonomous Reconfigurable Network of Wireless Robots, SWANBOT, than can adapt its configuration over time to maintain link qualities while performing some allocated task. As a part of our future work on this specific topic, we plan to develop a more formal algorithmic approach with polynomial time complexity as well as flesh out analytical details about the correctness of the bounds, if possible. Another direction of future work will be to validate this bounds with real testbed experiments.

\appendices

\section{Proof of Orthogonal Code Bound}
\label{App:AppendixO}
Say, at any time instance, the number of active interferers is $N_\mathcal{I} \in [0,N_\mathcal{I}^{max}]$. Given that $N_\mathcal{I}$ number of nodes are active and $N_\mathcal{O}\geq N_\mathcal{I}$, the probability of interference free communication is as follows.
{\footnotesize
\begin{equation}
\begin{split}
     &\mathbb{P}(\mathbbm{1}_{\mathcal{I}0} = 1|N_\mathcal{I}) = \frac{\comb{N_{\mathcal{O}}}{N_\mathcal{I}}\times N_\mathcal{I}!}{(N_{\mathcal{O}})^{N_\mathcal{I}}}
    =\prod_{i=1}^{N_\mathcal{I}} \left( 1- \frac{i-1}{N_{\mathcal{O}}} \right)\\
    \implies &\mathbb{P}(\mathbbm{1}_{\mathcal{I}0} = 1|N_\mathcal{I}^1)  \geq \mathbb{P}(\mathbbm{1}_{\mathcal{I}0} = 1|N_\mathcal{I}^2) \mbox{\ \ \ \ if $N_\mathcal{I}^1 \leq N_\mathcal{I}^2 \leq N_\mathcal{O}$ }
\end{split}
\label{eqn:mono}
\end{equation}
}
% Now, if $N_\mathcal{I}^1 \leq N_\mathcal{I}^2$ and $N_\mathcal{O} \geq N_\mathcal{I}^2 \geq N_\mathcal{I}^1 $, we can say that:
% {\footnotesize
% \begin{equation}
%     \prod_{i=1}^{N_\mathcal{I}^1} \left( 1- \frac{i-1}{N_{\mathcal{O}}} \right) \geq \prod_{i=1}^{N_\mathcal{I}^2} \left( 1- \frac{i-1}{N_{\mathcal{O}}} \right)  
%     \label{eqn:mono}
% \end{equation}
% }
Thus, the probability of interference free transmission for $N_\mathcal{O}\geq N^{max}$, where $N^{max}=N^{max}_{\mathcal{I}}+1$, can be expressed as follows.
{\footnotesize
\begin{equation}
    \begin{split}
        \mathbb{P}(\mathbbm{1}_{\mathcal{I}0} = 1) & = \sum_{j=0}^{N^{max}} \mathbb{P}(\mathbbm{1}_{\mathcal{I}0} = 1| N_\mathcal{I}=j)\mathbb{P}(N_\mathcal{I}=j) \\
        &=\sum_{j=0}^{N^{max}} \prod_{i=1}^{j } \left( 1- \frac{i-1}{N_{\mathcal{O}}} \right) \mathbb{P}(N_\mathcal{I}=j)\\
        & \geq \prod_{i=1}^{N^{max}} \left( 1- \frac{i-1}{N_{\mathcal{O}}} \right)  \sum_{j=0}^{N^{max}} \mathbb{P}(N_\mathcal{I}=j) \mbox{\ \ Using~\eqref{eqn:mono}}\\
        & \geq \prod_{i=1}^{N^{max}} \left( 1- \frac{i-1}{N_{\mathcal{O}}} \right)
        % &=\frac{N_{\mathcal{O}}!}{(N_{\mathcal{O}}- N^{max})! \times (N_{\mathcal{O}})^{N^{max}}}
    \end{split}
    \label{eqn:ortho_proof}
\end{equation}
}

\section{Proof of Lemma~\ref{lemma:cardinality}} 
\label{App:AppendixA}
% the \\ insures the section title is centered below the phrase: AppendixA
A valid Solution to the Pack Problem can be directly mapped to a valid Interference Set Cover.
To prove that, let us consider the set of centres, $S_P$,  for the circles in the Pack Problem solution.  For any valid solution to the Pack Problem, the distance between the centers of the circles are at least $R_1$ which satisfies the Interference Set Cover distance condition. Now, the center of any circle to be packed must lie in the annulus with radius $R_1$ and $R_2$ as the radius of the circles are $\frac{R_1}{2}$. Thus $S_P$ is a valid Interference Set Cover.
Next, assume the solution to the pack problem, $n$, does not contain maximum number of interferer. So there must exist an Interference Set Cover with more than $n$ interferer. However, if we formulate a set of circles with the centers to be same as the Interference Set Cover but with radius equal to $\frac{R_1}{2}$, it is also a valid circle packing solution with higher cardinality. This is a contradiction. Thus the earlier assumption is not true. Conversely, say that the solution to the Pack problem have higher cardinality than the max cardinality of Interference Set Cover, we can always map the Pack problem solution to a new Interference Set Cover with higher cardinality than the earlier solution. This is also a contradiction, thus, proves the lemma. 

% The proof of this lemma is two folds.
% First, we show that a valid Solution to the Pack Problem can be directly mapped to a valid Interference Set Cover.
% Let us consider the set of centres, $S_P$,  for the circles in the Pack Problem solution.  For any valid solution to the Pack Problem, the distance between the centers of the circles are at least $R_1$ which satisfies the Interference Set Cover distance condition. Now, the center of any circle to be packed must lie in the annulus with radius $R_1$ and $R_2$ as the radius of the circles are $\frac{R_1}{2}$. Thus $S_P$ is a valid Interference Set Cover.

% Next, we prove by contradiction, that the max cardinality of an Interference Set Cover is same as the maximum cardinality of the equivalent pack problem. Let's say that the max cardinality are not same and the solution to the pack problem, $n$, does not contain maximum number of interferer. So there must exist an Interference Set Cover with more than $n$ interferer. However, if we formulate a set of circles with the centers to be same as the Interference Set Cover but with radius equal to $\frac{R_1}{2}$, it is also a valid circle packing solution with higher cardinality. This is a contradiction. Conversely, if the solution to the Pack problem have higher cardinality than the max cardinality of Interference Set Cover, we can always map the Pack problem solution to a new  Interference Set Cover with higher cardinality than the earlier solution. Thus the earlier assumption is not true, which proves the lemma. 

\section{Proof of Lemma~\ref{lemma:2}}
\label{App:AppendixB}
% the \\ insures the section title is centered below the phrase: Appendix B
According to Assumption~\ref{assm:straight} states that in the final configuration, the routers should be placed along the line segment joining the sink and source, say $Line_{opt}$.. 
Now, assume that the first interferer in the worst case interference combination is located at $D_1+\delta$ distance from the source, along $Line_{opt}$, instead of $D_1$ where $0< \delta < (D_2-D_1)$.
Since, the distance between two interferer have to be greater than $D_1$ for concurrent transmission, the resulting set of interferers are located at $\mathcal{I}_1=\{D_1+\delta, 2*D_1+\delta,\cdots, k*D_1+\delta \}$ where $k*D_1+\delta \leq D_2$. 
% Assume that the transmit power of each node is same, $P_t$. The reason behind such positioning is because any nodes within $D_1$ distance from a interferer can not transmit simultaneously due to CSMA. 
Now, the Interference Power is inversely proportional to distance, more specifically $d^{-\eta}$ where $2\leq \eta \leq 6$ is the path loss exponent. Now say, the receiver is located at distance $d$ from the transmitter on the same side as the interferers. Therefore, the power of the interferer located at $D_1+\delta$ is less the the power of interferer located at $D_1$ as $\frac{1}{\left(D_1-d\right)^{\eta}}\geq \frac{1}{\left(D_1+\delta-d\right)^{\eta}}$. Similarly if the receiver is located at distance $d$ from the transmitter on the other side i.e, the distance between the first interferer and the receiver is $D_1+\delta+d$, the power of the interferer located at $D_1+\delta$ is less the the power of interferer located at $D_1$ as $\frac{1}{\left(D_1+d\right)^{\eta}}\geq \frac{1}{\left(D_1+\delta+d\right)^{\eta}}$
Thus, if we exchange the first interferer position with $D_1$ i.e, $\mathcal{I}_2=\{D_1, 2*D_1+\delta,\cdots, k*D_1+\delta \}$ where $k*D_1+\delta \leq D_2$ then we get set of location with total interference power higher than that of $I_1$. This is a contradiction. Thus the earlier assumption is wrong, thus, proves the lemma.
% We can similarly show that if we place the second interferer at  $2*D_1$ distance instead of $2*D_1+\delta \leq D_2$, we again get a new set with interference power more than $\mathcal{I}_2$ and so on. In summary, we can say that the optimal set of intra-flow interferers in terms of total interference power are located at distance $\{D_1,2*D_1,\cdots k*D_1\}$ from a transmitting node.

\section{Proof of Lemma~\ref{lemma:lines_choice}}
\label{App:AppendixD}
To prove this, we first introduce another lemma as follows.
\begin{lemma}
The length of the chords of an annulus with inner radius $D_1$ and outer radius $D_2$, located at $d_r< D_1$ distance from the centre increases monotonically with $d_r$. (Proof in Appendix~\ref{App:AppendixC}) \qed
\label{lemma:chord}
\end{lemma}
WLOG, we assume that $Y_2$ must be part of the interfering set cover.
Next, for proving this claim, we subdivide the angular region around point $Y_2$ into four regions, demonstrated in Figure~\ref{fig:new_2_flow_cover}. For region $I$ and $IV$ we can show that the maximum interference power from any flow, placed along any line in that region, is upper bounded by the interference power from a flow located on $l_{\mathcal{Z}}$ as shown in Figure~\ref{fig:new_2_flow_cover}. For any random line $l$ in Zone I, the next interfering nodes on either side of $Y_2$ are, say, $P_1$ and $P_2$ while the same for $l_Z$ are $Z_1,Z_3$, respectively. From triangular geometry, $||XP_1|| \geq ||XZ_1||$  as $||Y_2P_1||=||Y_2Z_1||=D_1$ whereas $||XP_2|| \geq ||XZ_3||$ (Due to the presence of node $Y_4$). Thus the interference power from $Z_1$ is greater than or equal to $P_1$, and the interference from $Z_3$ is greater or equal to the interference from $P_2$. This way we can show that the maximum interference power from a flow located along $l_{\mathcal{Z}}$ is always ahead of the same for $l$ with same number of interferer on either side of $Y_2$.  Furthermore, using the properties of an annulus along with Lemma~\ref{lemma:chord}, it can be easily shown that the length of $l$ is less than the length of $l_{\mathcal{Z}}$ and therefore can support less number of simultaneously interfering nodes than $l_{\mathcal{Z}}$. Thus, the maximum interference power from a flow on $l$ is less than the maximum interference power from a flow on $l_{\mathcal{Z}}$. Due to symmetry, we can similarly prove that the interference power from a flow located along any line $l$ in Zone IV is always upper bounded by the maximum interference power from a flow located along $l_{\mathcal{Z}}$.
% {\footnotesize
% \begin{equation}
% \begin{split}
%     P_{ll}=\frac{P_t}{((\frac{R_1}{2}-d)^2+ (\frac{\sqrt{3}}{2}R_1+R_1)^2)}^{\frac{\eta}{2}}+ \
%     \frac{P_t}{((\frac{R_1}{2}-d)^2+ (\frac{\sqrt{3}}{2}R_1)^2)}^{\frac{\eta}{2}} \\
%     P_{l_1}=\frac{P_t}{((\frac{R_1}{2}+d)^2+ (\frac{\sqrt{3}}{2}R_1)^2)}^{\frac{\eta}{2}}+ \
%     \frac{P_t}{((\frac{R_1}{2}+R-d)^2+ (\frac{\sqrt{3}}{2}R_1)^2)}^{\frac{\eta}{2}}
% \end{split}
% \end{equation}
% }

Now, for region II and III, we claim that interference power from a flow located along any random line segment $l$ is always upper bounded by the maximum interference power of a flow located along $l_{\mathcal{W}}$. In such cases, the power from $P_2$ is less than the power from $Y_3$, whereas the power from $P_1$ is greater than the power from $W_1$, or vice versa. Thus, there is no straight forward dominance of the power from either line segment. Instead the sum of the power dominates for $l_{\mathcal{W}}$. To show this, we perform a brute force simulation algorithm where we first add up the total interference power from $Y_3$ and $W_1$, and $P_1$ and $P_2$, respectively, which verified that the former is always higher than later. Similarly, we perform simulation to show that the maximum interference power from a flow along $l$ is always upper bounded by the maximum interference power from a flow along the line $l_\mathcal{W}$.

% Combining these two facts, we argue that the total interference power from a single interfering flow is always upper bounded by the maximum of the maximum interference powers from a flow located along $l_\mathcal{Z}$ and $l_\mathcal{W}$, respectively. Note that the line $l_{\mathcal{Z}}$ is same as the line $l_0$ in Figure~\ref{fig:inter_cover}.

% \todo{It can be also shown than the maximum interference power from a flow along any random line segment on the disk is also upper bounded by the maximum of the worst case interference power from a flow located along $l_\mathcal{Z}$ and $l_\mathcal{W}$, respectively (Under Assumption~\ref{assm:1}). We omit the details to meet the page requirements.}
\section{Proof of Lemma~\ref{lemma:chord}}
\label{App:AppendixC}
Lets take a random chord of the annulus, located at $d_r$ distance from the center with $d_r <D_1$. Then the length of the chord is equal to $g(d_r)=\sqrt{D_2^2-d_r^2}-\sqrt{D_1^2-d_r^2}$. Now taking derivative of $g(.)$ as follows.
{\footnotesize
\begin{equation}
    \begin{split}
        g'(d_r)&=-\frac{d_r}{\sqrt{D_2^2-d_r^2}}+\frac{d_r}{\sqrt{D_1^2-d_r^2}}\\
        &=-\frac{1}{\sqrt{(\frac{D_2}{d_r})^2-1}}+\frac{1}{\sqrt{(\frac{D_1}{d_r})^2-1}}\\
        &> 0 \ \mbox{as}\ D_2>D_1 \ \mbox{and} \ d_r<D_1
    \end{split}
\end{equation}
}
This implies that $g(.)$ is a strictly increasing function of $d_r$, which proves our lemma.
\ifCLASSOPTIONcaptionsoff
  \newpage
\fi

\balance

{
 \bibliographystyle{unsrt}
\bibliography{ref}
}


%\vfill


\end{document}